\documentclass[twoside]{article}

\usepackage[accepted]{aistats2021}

\usepackage[round]{natbib}

\usepackage{booktabs} 
\usepackage{enumitem}
\usepackage{mdframed}
\usepackage{bm}
\usepackage{hyperref}
\usepackage{amssymb,amsmath}
\usepackage{mathtools}
\mathtoolsset{showonlyrefs}
\usepackage{amsthm}

\usepackage{algorithm}
\usepackage[noend]{algpseudocode}
\algrenewcommand{\algorithmiccomment}[1]{$\vartriangleright$ #1}
\algrenewcommand{\algorithmicreturn}{\textbf{Return: }}
\algnewcommand\algorithmicinput{\textbf{Input: }}
\algnewcommand\Input{\State \algorithmicinput}

\usepackage{todonotes}

\DeclareMathOperator*{\argmin}{argmin}
\DeclareMathOperator*{\diag}{diag}
\def\vect{\text{vec}}

\def\dtw{\textsc{dtw}}
\def\OT{\textsc{OT}}
\def\sdtw{\textsc{sdtw}}
\def\sdtwg{\textsc{sdtw}_\gamma}
\def\sharpg{\textsc{sharp}_\gamma}
\def\smin{\text{min}}
\def\dgC{D_\gamma^C}
\def\sgC{S_\gamma^C}
\def\meancost{\text{\textsc{mean\_cost}}}
\def\Delannoy{\text{\normalfont{Delannoy}}}

\def\p{\bm{p}}

\def\A{\bm{A}}
\def\B{\bm{B}}
\def\C{{\bm{C}}}
\def\E{\bm{E}}

\def\K{\bm{K}}
\def\M{\bm{M}}
\def\P{\bm{P}}

\def\T{\bm{T}}
\def\V{{\bm{V}}}
\def\X{{\bm{X}}}
\def\Y{{\bm{Y}}}
\def\Z{{\bm{Z}}}

\def\bomega{\bm{\omega}}

\def\v{{\bm{v}}}
\def\t{{\bm{t}}}
\def\x{{\bm{x}}}
\def\y{{\bm{y}}}
\def\ones{{\bm{1}}}
\def\zeros{{\bm{0}}}
\def\balpha{\bm{\alpha}}
\def\bbeta{\bm{\beta}}

\def\cA{\mathcal{A}}
\def\cS{\mathcal{S}}
\def\cU{\mathcal{U}}

\def\NN{\mathbb{N}}
\def\RR{\mathbb{R}}
\def\EE{\mathbb{E}}
\def\PP{\mathbb{P}}
\def\Rmn{\RR^{m \times n}}
\def\Rmd{\RR^{m \times d}}
\def\Rnd{\RR^{n \times d}}
\def\Rmm{\RR^{m \times m}}

\newcommand{\partialfrac}[2]{\frac{\partial #1}{\partial #2}}

\newmdtheoremenv{proposition}{Proposition}
\newmdtheoremenv{lemma}{Lemma}
\newmdtheoremenv{corollary}{Corollary}

\usepackage{color}
\usepackage{xcolor}

\colorlet{darkblue}{blue!80!black}
\colorlet{darkred}{red!80!black}
\colorlet{darkgreen}{green!60!black}
\colorlet{darkmagenta}{orange!80!black}
\colorlet{darkyellow}{purple!80!black}

\newcommand{\mygray}[1]{\textcolor{gray}{#1}}

\hypersetup{
    colorlinks,
    linkcolor={black},
    citecolor={blue!50!black},
    urlcolor={blue!50!black}
}

\begin{document}



\twocolumn[

\aistatstitle{Differentiable Divergences Between Time Series}

\aistatsauthor{ Mathieu Blondel \And Arthur Mensch \And Jean-Philippe Vert}

\aistatsaddress{ Google Research, Brain team \And \'{E}cole Normale
Sup\'{e}rieure \And Google Research, Brain team } 
]

\begin{abstract}
Computing the discrepancy between time series of variable sizes is notoriously
challenging.  While dynamic time warping (DTW) is popularly used for this
purpose, it is not differentiable everywhere and is known to lead to bad local
optima when used as a ``loss''. Soft-DTW addresses these issues,
but it is not a positive definite divergence: due to the bias introduced by
entropic regularization, it can be negative and it is not minimized when the 
time series are equal. 
We propose in this paper a new divergence, dubbed soft-DTW divergence, which
aims to correct these issues. We study its properties; in
particular, under conditions on the ground cost, we show that it is
a valid divergence: it is non-negative and minimized if and
only if the two time series are equal. We
also propose a new ``sharp'' variant by further removing entropic bias.  
We showcase our divergences on time series averaging and demonstrate significant
accuracy improvements compared to both DTW and soft-DTW on 84 time series
classification datasets.
\end{abstract}

\section{Introduction}

Designing a meaningful discrepancy or ``loss'' between two sequences of variable
lengths and integrating it in an end-to-end differentiable pipeline is
challenging. 
For sequences on finite alphabets, differentiable local alignment kernels
\citep{saigo_2006} and edit distances \citep{mccallum_2012} have been proposed.
For sequences on continuous domains, connectionist temporal classification
(CTC) is popularly used in speech recognition \citep{graves_2006}. 
A related approach for time series motivated by geometry is dynamic time
warping (DTW), which seeks a minimum-cost alignment between time series
and can be computed by dynamic programming in quadratic time \citep{sakoe_1978}.
However, DTW is not differentiable everywhere, is sensitive to noise and is
known to lead to bad local optima when used as a loss.  Soft-DTW
\citep{soft_dtw} addresses these issues by replacing the minimum over
alignments with a soft minimum, which has the effect of inducing a probability
distribution over all alignments. Despite considering all alignments, it is
shown that soft-DTW can still be computed by dynamic programming in the same
complexity. Since then, soft-DTW has been successfully applied for
audio to music score alignment \citep{diff_dp},
video segmentation \citep{d3tw},
spatial-temporal sequences \citep{janati_2020},
and end-to-end differentiable text-to-speech synthesis \citep{donahue_2020}, to
name but a few examples.
Soft-DTW is included in popular R and Python packages for time series
analysis \citep{sarda_2017,tavenard_2020}.

In this paper, we show that, despite recent successes, soft-DTW has some
limitations which have been overlooked in the literature. First, it can be
negative, which is a nuisance when used as a loss. Second, and more problematically, 
when used with a squared Euclidean cost, we show
that it is never
minimized when the two time series are equal. Put differently, given an input time
series, the closest time series in the soft-DTW sense is never the input time
series. This is due to the entropic bias introduced by replacing the minimum
with a soft one. We propose in this paper a new divergence, dubbed soft-DTW
divergence, which is based on soft-DTW but corrects for these issues.
We study its properties; in particular, under condition on
the ground cost, we show that it is a valid divergence:
it is non-negative and it is minimized if and only if the two time series are equal.
Our approach is related to Sinkhorn divergences
\citep{ramdas_2017,genevay_2018,feydy_2019}, which use similar correction terms
as we do for optimal transport distances, but our proof techniques are
completely different. We also propose a new ``sharp''
variant by further removing entropic bias.
We showcase our divergences on time series averaging and demonstrate
significant accuracy improvements compared to both DTW and soft-DTW on 84 time
series classification datasets.

The rest of the paper is organized as follows. After reviewing some background
in \S\ref{sec:background}, we introduce the soft-DTW divergence and its
``sharp'' variant in \S\ref{sec:proposed_div}. We study their properties and
limit behavior. We study their empirical performance in \S\ref{sec:exp} with
experiments on time series averaging, interpolation and classification.

\section{Background}
\label{sec:background}

\subsection{Dynamic time warping}

Let $\X \in \RR^{m \times d}$ and $\Y \in \RR^{n \times d}$ be two $d$-dimensional time
series of lengths $m$ and $n$. We denote their elements by $\x_i \in \RR^d$
and $\y_j \in \RR^d$, for $i \in [m]$ and $j \in [n]$. We say that $\A \in
\{0,1\}^{m \times n}$ is an alignment matrix between $\X$ and $\Y$ when
$[\A]_{i,j} = 1$ if $\x_i$ is aligned with $\y_j$ and $0$ otherwise.
We say that $\A$ is a monotonic alignment matrix if the ones in $\A$ form a path starting
from the upper-left corner $(1,1)$ that connects the lower-right corner $(m,n)$
using only $\downarrow$, $\rightarrow$, $\searrow$ moves.
We denote the set of all such monotonic alignment matrices by 
$\cA(m,n) \subset \{0,1\}^{m \times n}$.
The cardinality $|\cA(m,n)|$ grows exponentially in $\min(m,n)$ 
and is equal to the Delannoy number, $\text{Delannoy}(m-1,n-1)$, named after
French amateur mathematician Henri Delannoy
\citep{sulanke_2003,banderier_2005}.

Let $C \colon \RR^{m \times d} \times \RR^{n \times d} \to \RR^{m \times n}$ be
a function which maps $\X \in \RR^{m \times d}$ and $\Y \in \RR^{n \times d}$ to
a distance or cost matrix $\C = C(\X, \Y) \in \RR^{m \times n}$.
A popular choice is the squared Euclidean cost
\begin{equation}
[C(\X, \Y)]_{i,j} = \frac{1}{2} \|\x_i - \y_j\|^2_2
\quad i \in [m], j \in [n].
\label{eq:squared_euclidean}
\end{equation}
The Frobenius inner product
$\langle \A, \C \rangle \coloneqq \text{Trace}(\C^\top \A)$ 
between $\C$ and $\A$ is the sum of the costs along the alignment (Figure \ref{fig:dag}).
Dynamic time warping \citep{sakoe_1978} can then be naturally formulated as the
minimum cost among all possible alignments,
\begin{equation}
\dtw(\C) \coloneqq
\min_{\A \in \cA(m,n)} \langle \A, \C \rangle.
\label{eq:dtw}
\end{equation}
The corresponding optimal alignment (not necessarily unique) is then
\begin{equation}
\A^\star(\C) \in
\argmin_{\A \in \cA(m,n)} \langle \A, \C \rangle.
\label{eq:optimal_alignment}
\end{equation}
Despite the exponential number of alignments,
\eqref{eq:dtw} and \eqref{eq:optimal_alignment} can be computed in $O(mn)$ time
using dynamic programming and backtracking, respectively. 
The quantity $\dtw(C(\X, \Y))$ is popularly used as a
discrepancy measure between time series in numerous applications. 
In the rest of the paper, we will make the following assumptions about
the ground cost $C$:
\begin{itemize}[topsep=0pt,itemsep=3pt,parsep=2pt,leftmargin=9pt]
    \item A.1. $C(\X, \Y) \ge \zeros_{m \times n}$ (non-negativity),
    \item A.2. $[C(\X, \X)]_{i,i} = 0$ for all $i \in [m]$,
    \item A.3. $C(\X, \Y) = C(\Y, \X)^\top$ (symmetry).
\end{itemize}
The properties of $\dtw$ under these assumptions are summarized in Table
\ref{table:properties}. Note that $\dtw$ is minimized at $\X = \Y$ but this
may not be the unique minimum.

\begin{table*}[t]
\caption{Properties of time-series losses under assumptions
A.1-A.3 and differentiability of $C$.
For the soft-DTW divergence, we prove non-negativity and ``minimized at
$\X=\Y$'' using the cost
\eqref{eq:log_cost} and one-dimensional absolute value \eqref{eq:lap_cost}
(cf.\ Proposition \ref{prop:non_negativity_log}). 
For the soft-DTW and sharp divergences with the squared Euclidean cost
\eqref{eq:squared_euclidean}, we only prove that $\X=\Y$ is a stationary point
(cf.\ Proposition \ref{prop:stationary_point})
}
\label{table:properties}
\centering
\begin{tabular}{lcccc}
\toprule
& Non-negativity & Minimized at $\X=\Y$ & Symmetry & Differentiable everywhere \\
\midrule
DTW & $\checkmark$ & $\checkmark$ & $\checkmark$ & $\times$ \\[0.3em]
Soft-DTW & $\times$ & $\times$ & $\checkmark$ & $\checkmark$ \\[0.3em]
Sharp soft-DTW & $\checkmark$ & $\times$ & $\checkmark$ & $\checkmark$ \\[0.3em]
Soft-DTW divergence & $\checkmark$ & $\checkmark$ & $\checkmark$ & $\checkmark$ \\[0.3em]
Sharp divergence & $\checkmark$ & $\checkmark$ & $\checkmark$ & $\checkmark$ \\[0.3em]
Mean-cost divergence & $\checkmark$ & $\checkmark$ & $\checkmark$ & $\checkmark$ \\[0.3em]
\hline
\end{tabular}
\end{table*}

\begin{figure}
\centering
\includegraphics[width=0.48\textwidth]{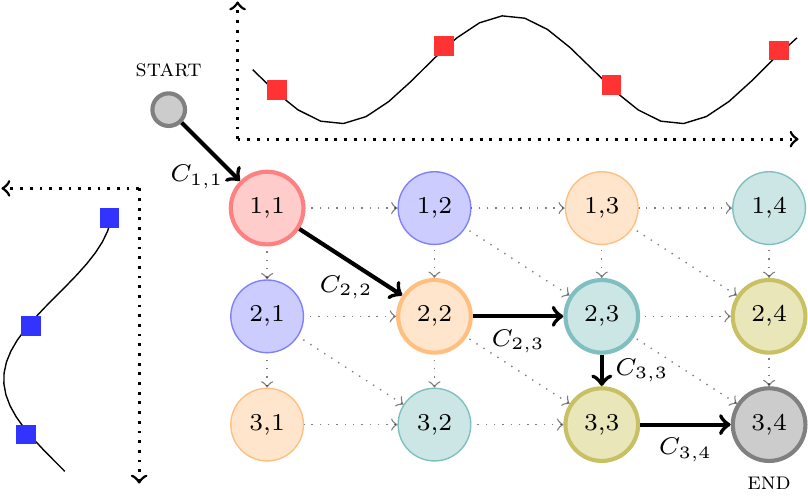}
\caption{An alignment between two time series 
$\X \in \Rmd$ and $\Y \in \Rnd$ corresponds to a path in a directed acyclic
graph (DAG) and can be encoded as a binary matrix $\A \in \{0,1\}^{m \times n}$.
The sum of the costs along the path is then $\langle \A, \C \rangle$. DTW seeks
a minimum cost alignment, while soft-DTW seeks the soft minimum cost alignment.
The latter induces a Gibbs distribution over all alignments.}
\label{fig:dag}
\end{figure}

\subsection{Soft dynamic time warping}
\label{sec:sdtw}

\paragraph{Definitions.}

In order to obtain a fully differentiable discrepancy measure between time
series, \citet{soft_dtw} proposed to replace the min operator in \eqref{eq:dtw}
by a smooth one,
\begin{equation}
\underset{x \in \cS}{\smin_\gamma} ~ f(x) 
\coloneqq -\gamma \log \sum_{x \in \cS} \exp(-f(x) / \gamma),
\end{equation}
where $\gamma > 0$ is a parameter which controls the trade-off between
approximation and smoothness. 
For convenience, we define the extension $\smin_0 \coloneqq \min$.
The resulting ``soft'' dynamic time warping
formulation is
\begin{align}
\sdtwg(\C) 
&\coloneqq \underset{\A \in \cA(m,n)}{\smin_\gamma} \langle \A, \C \rangle
\\
&= -\gamma \log \sum_{\A \in \cA(m,n)} \exp(-\langle \A, \C \rangle /
\gamma).
\label{eq:sdtw}
\end{align}
Instead of only considering the minimum-cost alignment as in \eqref{eq:dtw},
\eqref{eq:sdtw} induces a Gibbs distribution over alignments.
The probability of $\A$ given $\C \in \RR^{m \times n}$ is
\begin{equation}
\PP_\gamma(\A; \C) 
\coloneqq \frac{\exp(-\langle \A, \C \rangle / \gamma)}
{\sum_{\A' \in \cA(m,n)} \langle -\langle \A', \C \rangle / \gamma)}
\in (0,1].
\label{eq:proba}
\end{equation}
We can see \eqref{eq:sdtw} as the negative log-partition of \eqref{eq:proba}.
For convenience, we also gather the probabilities of all possible alignments in
a vector
\begin{equation}
\p_\gamma(\C) \coloneqq (\PP_\gamma(\A; \C))_{\A \in \cA(m,n)} \in
\triangle^{|A(m,n)|},
\end{equation}
where $\triangle^k \coloneqq \{\p \in \RR^k \colon \p \ge \zeros_k, \p^\top
\ones_k = 1\}$ is the probability simplex.
Let $A$ be a random variable distributed according to \eqref{eq:proba}.
The expected alignment matrix under the Gibbs distribution induced by $\C$ is
\begin{equation}
\E_\gamma(\C) 
\coloneqq \EE_\gamma[A; \C] 
= \sum_{\A \in \cA(m,n)} \PP_\gamma(\A; \C) \A
\in (0,1]^{m \times n}.
\label{eq:expectation}
\end{equation}
Note that because the matrices in $\cA(m,n)$ are binary ones,
$[\E_\gamma(\C)]_{i,j}$ is also equal to the marginal probability
$\PP_\gamma(A_{i,j}=1; \C)$, i.e., the probability that any of the paths goes
through the cell $(i,j)$.

\paragraph{Computation.}

Surprisingly, even though \eqref{eq:sdtw} contains a sum over all $\A$ in
$\cA(m,n)$, it can be computed in $O(mn)$ time by simply replacing the $\min$
operator with $\smin_\gamma$ in the original dynamic programming recursion
\citep{soft_dtw}. 
See also Algorithm \ref{alg:sdtw_value} in Appendix \ref{appendix:algorithms}.  
The equivalence between \eqref{eq:sdtw} and this ``locally
smoothed'' recursion was later formally proved using the associativity of the
$\smin_\gamma$ operator \citep{diff_dp}. The expected alignment can also be
computed in $O(mn)$ time by backpropagation through the dynamic programming
recursion \citep{soft_dtw}.
See also Algorithm \ref{alg:sdtw_gradient} in Appendix \ref{appendix:algorithms}.  

\paragraph{Properties.}

The following proposition summarizes known properties of $\sdtwg$
\citep{soft_dtw,diff_dp}.
\begin{proposition}{Properties of $\sdtwg$}

The following properties hold for all $\C \in \RR^{m \times n}$.
\begin{enumerate}[topsep=0pt,itemsep=3pt,parsep=2pt,leftmargin=9pt]
\item {\bf Gradient:} $\sdtwg(\C)$ is differentiable everywhere and its gradient
is the expected alignment,
\begin{equation}
\nabla_\C \sdtw_\gamma(\C) = \E_\gamma(\C) \in (0,1]^{m \times n}.
\end{equation}

\item {\bf Concavity:} $\sdtwg(\C)$ is concave in $\C$. 

\item {\bf Variational form:} 
letting $H(\p) = -\langle \p, \log \p \rangle$,
\begin{equation}
\sdtwg(\C) = \min_{\p \in \triangle^{|\cA(m,n)|}} \langle 
\p, \bm{s}(\C) \rangle - \gamma H(\p)
\label{eq:varitional_form}
\end{equation}
where $\bm{s}(\C) \coloneqq (\langle \A, \C \rangle)_{\A \in \cA(mmn)} \in
\RR^{|\cA(m,n)|}$.

\item {\bf Scaling:} $\sdtwg(\C) = \gamma \sdtw_1(\C / \gamma)$,
$\E_\gamma(\C) = \E_1(\C / \gamma)$ and $\p_\gamma(\C) = \p_1(\C / \gamma)$.

\item {\bf Asymptotics:} $\dtw(\C) \xleftarrow[0 \leftarrow \gamma]{}
\sdtwg(\C)$ and $\A^\star(\C) \xleftarrow[0 \leftarrow \gamma]{}
\E_\gamma(\C)$.

\item {\bf Lower and upper bounds:}
\begin{equation}
\small
\dtw(\C) - \gamma \log |\cA(m,n)| 
\le \sdtw_\gamma(\C)
\le \dtw(\C).
\end{equation}
\end{enumerate}
\label{prop:sdtw_properties}
\end{proposition}
Note that $\sdtwg(C(\X, \Y))$ is generally neither convex nor concave in $\X$
and $\Y$, as is the case when $C$ is the squared Euclidean cost
\eqref{eq:squared_euclidean}. A notable exception is $C(\X, \Y) = -\X \Y^\top$,
for which $\sdtwg(C(\X, \Y))$ is concave in $\X$ and $\Y$ (separately).

\paragraph{Use as a loss function.}

The differentiability of $\sdtwg$ makes it particularly suitable
to use as a loss function between time series, of potentially variable lengths.
An example of application is the computation of Fr\'{e}chet means
(\citeyear{frechet_1948}) with respect to $\sdtwg$. Specifically,
given a set of $k$ time series $\Y_1 \in \RR^{n_1 \times d}$,
$\dots$, $\Y_k \in \RR^{n_k \times d}$, we compute its 
average (barycenter) according to $\sdtwg$ by solving
\begin{equation}
\argmin_{\X \in \RR^{m \times d}} \sum_{i=1}^k w_i ~ \sdtwg(C(\X, \Y_i)),
\label{eq:barycenter_obj}
\end{equation}
where $\bm{w} = (w_1, \dots, w_k) \in \RR^k$ is a vector of pre-defined weights. 
When the time
series $\Y_1,\dots,\Y_k$ have different lengths, a typical choice would be $w_i
= 1/n_i$, to compensate for the fact that $\sdtwg$ increases 
with the length of the time series.  Although it is non-convex, objective
\eqref{eq:barycenter_obj} can be solved approximately by gradient-based methods.
Compared to DTW barycenter averaging (DBA) \citep{petitjean_2011}, it was shown
that smoothing helps to avoid bad local optima. Using the chain
rule and item 1 of Proposition \ref{prop:sdtw_properties}, the gradient of
$\sdtwg(C(\X, \Y))$ w.r.t.\ $\X$ is
\begin{equation}
\nabla_\X \sdtwg(C(\X, \Y)) =
(J_\X C(\X, \Y))^\top 
\E_\gamma(\C(\X, \Y)).
\label{eq:sdtw_grad_X}
\end{equation}
Here, we assume that $C$ is differentiable and $J_\X$ denotes the Jacobian
matrix of $\C(\X, \Y)$ w.r.t.\ $\X$, a linear map from $\RR^{m \times d}$ to
$\RR^{m \times n}$ (its transpose is a linear map from $\RR^{m \times n}$ to
$\RR^{m \times d}$). 

\subsection{Global alignment kernel}

Although it was introduced before soft dynamic time warping, the global
alignment kernel \citep{cuturi_2007} can be naturally expressed using $\sdtwg$
as
\begin{equation}
K_\gamma^C(\X, \Y) \coloneqq \exp(-\sdtw_1(C(\X, \Y) / \gamma)).
\label{eq:ga_kernel}
\end{equation}
Using a constructive proof, it was shown that
\eqref{eq:ga_kernel} is a positive definite (p.d.) kernel
under certain cost functions and in particular with
\begin{equation}
[C(\X, \Y)]_{i,j} = 
\delta(\x_i, \y_i) + \log(2 - \exp(-\delta(\x_i, \y_i)),
\label{eq:log_cost}
\end{equation}
where $\delta(\x, \y) \coloneqq \frac{1}{2} \|\x - \y\|^2_2$. 
In the one-dimensional case
($d=1$), we show in Appendix \ref{app:laplacian} that 
\begin{equation}
[C(\X, \Y)]_{i,j} = \|\x_i - \y_j\|_1,
\label{eq:lap_cost}
\end{equation}
also has the property that the kernel \eqref{eq:ga_kernel} is p.d. Using these
costs, \eqref{eq:ga_kernel} can be used in any kernel method, such as support
vector machines. The positive definiteness of \eqref{eq:ga_kernel} using the
squared Euclidean cost \eqref{eq:squared_euclidean} has to our knowledge not
been proved or disproved yet. 

\begin{figure}
\centering
\includegraphics[scale=0.45]{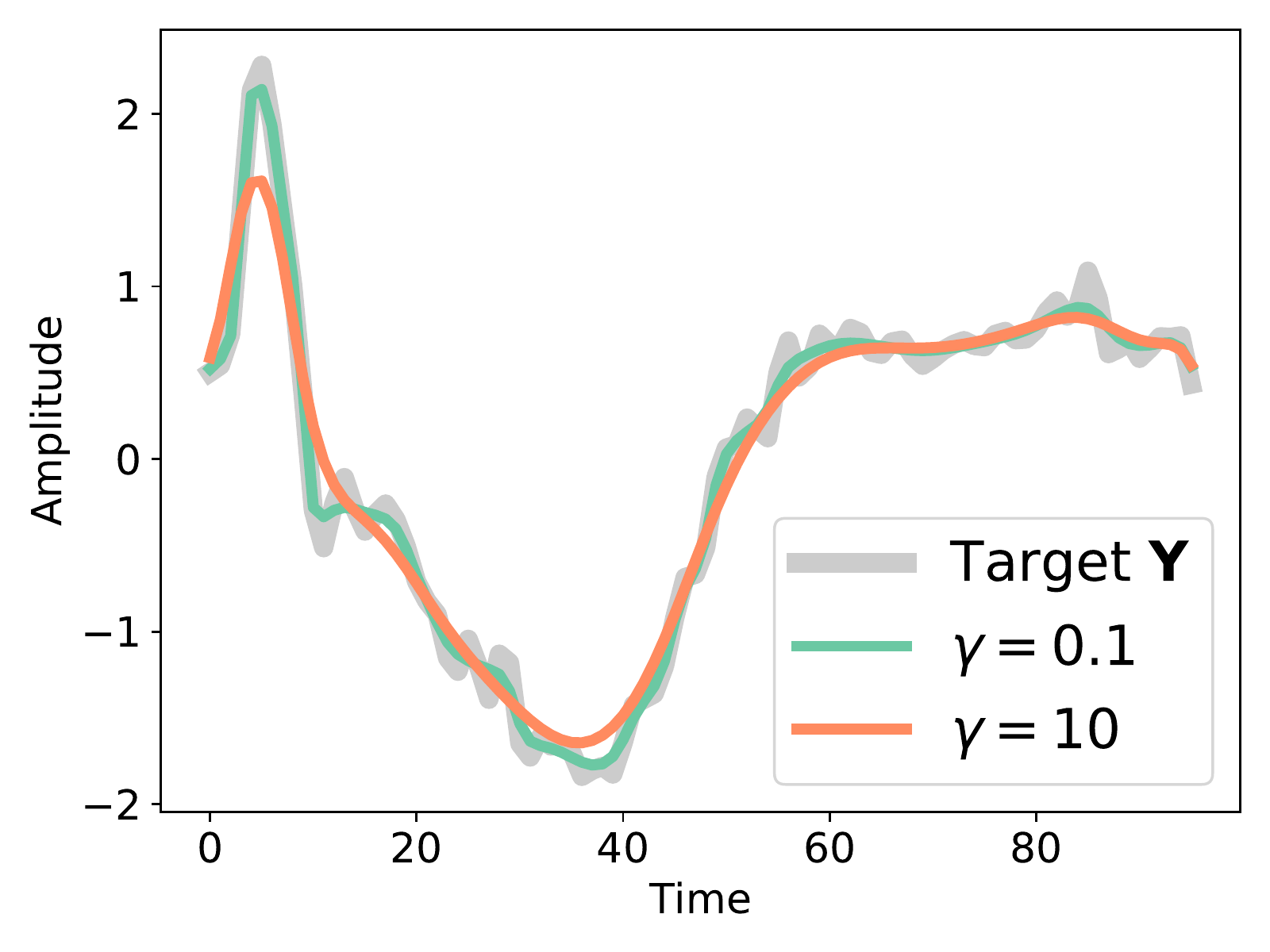}
\caption{{\bf Denoising effect of soft-DTW.} We show the result of $\argmin_{\X}
\sdtwg(C(\X, \Y))$, solved by L-BFGS with $\X=\Y$ as initialization, for two values
of $\gamma$. As stated in Proposition \ref{prop:limitations}, $\sdtwg$ with
$\gamma > 0$ and squared Euclidean cost never achieves its minimum at $\X=\Y$. 
While this denoising can be useful, this means that
$\sdtwg$ is not a valid divergence.}
\label{fig:denoising}
\end{figure}

\section{New differentiable divergences}
\label{sec:proposed_div}

In this section, we begin by pointing out potential limitations of soft-DTW. We
then introduce two new divergences, the soft-DTW divergence and its sharp
variant, which aim to correct for these limitations. We study their properties
and limit behavior.

\paragraph{Limitations of soft-DTW.}

Despite recent empirical successes, soft-DTW has some inherent limitations that
were not discussed in previous works. The following proposition clarifies
these limitations.
\begin{proposition}{Limitations of $\sdtwg$}

The following holds.
\begin{enumerate}[topsep=0pt,itemsep=3pt,parsep=2pt,leftmargin=9pt]

\item For all $\C \in \RR^{m \times n}$, $\gamma \mapsto
\sdtwg(\C)$ is non-increasing, concave, and diverges to $-\infty$ when
$\gamma\rightarrow +\infty$. In particular, there exists $\gamma_0 \in [0,
\infty)$ such that $\sdtwg(\C) \le 0$ for all $\gamma \ge \gamma_0$.

\item For all cost functions $C$ satisfying A.2, $\X\in\RR^{m\times d}$ and
$\gamma \in [0, \infty)$, $\sdtwg(C(\X, \X)) \le 0$. 

\item For the squared Euclidean cost \eqref{eq:squared_euclidean} and any
$\gamma \in (0, \infty)$, the minimum of $\sdtwg(C(\X, \Y))$ is not achieved
at $\X = \Y$. 

\end{enumerate}
\label{prop:limitations}
\end{proposition}
A proof is given in Appendix \ref{appendix:proof_limitations}.
Proposition \ref{prop:limitations} shows that that there exists values of
$\gamma$ or $\C$ for which $\sdtwg(\C)$ is negative.
Non-negativity is a useful property of divergences and the fact that 
$\sdtwg$ does not satisfy it can be a nuisance. More problematic is the fact
that $\sdtwg(C(\X, \Y))$ is not minimized at $\X = \Y$. 
This is illustrated in Figure \ref{fig:denoising}.
While the denoising effect of soft-DTW can be useful, we would expect a
proper differentiable divergence to be zero when the two time series are equal.

\paragraph{Soft-DTW divergences.}

To address these issues, we propose to use
for all $\X \in \Rmd$ and $\Y \in \Rnd$
\begin{align}
\dgC(\X, \Y) \coloneqq ~
&\sdtwg(C(\X, \Y)) \\
-\frac{1}{2} &\sdtwg(C(\X, \X)) \\
-\frac{1}{2} &\sdtwg(C(\Y, \Y)).
\end{align}
Since it is based on soft-DTW, we call it the soft-DTW divergence.
Sinkhorn divergences \citep{ramdas_2017,genevay_2018,feydy_2019},
which are divergences between probability measures based on entropy-regularized
optimal transport, use similar correction terms.

\paragraph{Sharp divergences.}

The variational form of $\sdtwg$ (Proposition \ref{prop:sdtw_properties})
implies that it can be decomposed as the sum of a cost term
and an entropy term,
\begin{equation}
\sdtwg(\C) = \langle \E_\gamma(\C), \C \rangle - \gamma H(\p_\gamma(\C)).
\label{eq:decomposition}
\end{equation}
On the other hand, we have
\begin{equation}
\dtw(\C) = \langle \A^\star(\C), \C \rangle.
\end{equation}
Since $\E_\gamma(\C) \to \A^\star(\C)$ when $\gamma \to 0$,
this suggests a new discrepancy measure,
\begin{equation}
\sharpg(\C) 
\coloneqq \langle \E_\gamma(\C), \C \rangle.
\label{eq:sharp}
\end{equation}
It is the directional derivative of 
$\sdtwg(\C)$ in the direction of $\C$, since
$\E_\gamma(\C) = \nabla_\C \sdtwg(\C)$.
Inspired by \citet{luise_2018}, who studied a similar idea in an optimal
transport context,we call it sharp soft-DTW, since it removes the entropic
regularization term $-\gamma H(\p_\gamma(\C))$ from \eqref{eq:decomposition}.
Its gradient is equal to
\begin{equation}
\nabla_\C \sharpg(\C) 
= \E_\gamma(\C) + \frac{1}{\gamma} \nabla^2_\C \sdtwg(\C) \C
\in \Rmn,
\label{eq:sharp_gradient}
\end{equation}
where $\nabla^2_\C \sdtwg(\C) \C$ is a Hessian-vector product (that can be
computed efficiently, as we detail below).  The gradient w.r.t.\ $\X$ is
obtained by the chain rule, similarly to \eqref{eq:sdtw_grad_X}.  Although
$\sharpg$ is trivially non-negative, it suffers from the same issue as $\sdtwg$,
namely, $\sharpg(C(\X, \Y))$ is not minimized at $\X = \Y$. We therefore propose
to use instead
\begin{align}
\sgC(\X, \Y) \coloneqq ~
&\sharpg(C(\X, \Y)) \\
-\frac{1}{2} &\sharpg(C(\X, \X)) \\
-\frac{1}{2} &\sharpg(C(\Y, \Y)).
\end{align}
We call it the sharp soft-DTW divergence.

\paragraph{Validity.}

We remind the reader that in mathematics, a \emph{divergence} $D$ is a function
that is non-negative ($D(\X,\Y)\geq 0$ for any $\X,\Y$) and that satisfies the
identify of indiscernibles ($D(\X,\Y)= 0$ if and only if $\X=\Y$).
By construction, we have $\dgC(\X, \X) = 0$ and $\sgC(\X,\X)=0$ for all $\X \in
\Rmd$.  Moreover, the following result shows that $\dgC$ is a valid divergence,
under some assumptions on the cost $C$.
\begin{proposition}{Valid divergence.} 
\label{prop:non_negativity_log}

Let $\gamma > 0$.
If $C$ is the cost defined in \eqref{eq:log_cost} with $d \in \NN$, or, 
if $C$ is the absolute value \eqref{eq:lap_cost} with $d=1$, then
$\dgC(\X, \Y) \ge 0$ for all $\X \in \Rmd$ and $\Y \in \Rnd$, and $\dgC(\X, \Y)
= 0$ if and only if $\X=\Y$. Therefore, $\dgC$ is a valid divergence. 
\end{proposition}
A proof is given in Appendix \ref{appendix:proof_nn_log_cost}.  This implies
that, for the costs \eqref{eq:log_cost} and \eqref{eq:lap_cost}, $\dgC(\X, \Y)$
is uniquely minimized at $\X = \Y$.  
The proof relies on the fact that the global alignment kernel
\eqref{eq:ga_kernel} is positive definite under these costs.
Unfortunately, since the positive definiteness of
\eqref{eq:ga_kernel} under the squared Euclidean cost
\eqref{eq:squared_euclidean} has not been proved or
disproved, the same proof technique does not apply.
Nevertheless, we can prove the following.
\begin{proposition}{Stationary point under cost \eqref{eq:squared_euclidean}}
    
If $C$ is the squared Euclidean cost \eqref{eq:squared_euclidean}, then $\X =
\Y$ is a stationary point
of $\dgC(\X, \Y)$ and $\sgC(\X,\Y)$ 
w.r.t.\ $\X \in \Rnd$ for all $\Y \in \Rnd$.
\label{prop:stationary_point}
\end{proposition}
A proof is given in Appendix \ref{appendix:proof_stationary_point}. Based
on Proposition \ref{prop:stationary_point} and ample numerical evidence (cf.
Appendix \ref{appendix:numerical_validation}), we
conjecture that $\dgC(\X,\Y)$ and $\sgC(\X, \Y)$ are also non-negative under the
squared Euclidean cost.

\vspace{-0.5em}
\paragraph{Asymptotic behavior.}

We now study the behavior of our divergences in the zero and infinite
temperature limits, i.e., when $\gamma \to 0$ and $\gamma \to \infty$.
As we saw, $\E_\gamma(\C)$ is the expected alignment matrix under the Gibbs
distribution $\PP_\gamma(\A; \C)$. Let $A$ be a random alignment matrix
\textit{uniformly} distributed over $\cA(m,n)$, i.e., independent of the cost
matrix $\C$.
Replacing $\E_\gamma(\C)$ with $\EE[A]$ in \eqref{eq:sharp}, we obtain the mean
cost, the average of the cost along all possible paths,
\begin{align}
\meancost(\C) 
&\coloneqq \langle \EE[A], \C \rangle \\
&= \frac{1}{|\cA(m,n)|} \sum_{\A \in \cA(m,n)} 
\langle \A, \C \rangle.
\label{eq:mean_cost}
\end{align}
We also define the mean-cost divergence,
\begin{align}
M^C(\X, \Y) 
&\coloneqq \meancost(C(\X, \Y)) \\
&- \frac{1}{2} \meancost(C(\X, \X)) \\
&- \frac{1}{2} \meancost(C(\Y, \Y)).
\end{align}
It bears some similarity with energy distances
\citep{baringhaus_2004,szekely_2004}, with the key difference that the
probability distribution is over the alignments, not over the time series.

We now show that our proposed divergences are all intimately related
through their asymptotic behavior, and that $\dgC$ and $\sgC$ share the same
limits to the right when $m=n$ but not when $m \neq n$.
\begin{proposition}{Limits w.r.t.\ $\gamma$}

For all $\C = C(\X, \Y) \in \Rmn$, $m=n$:
\begin{equation}
\dtw(\C) \xleftarrow[0 \leftarrow \gamma]{} \dgC(\X, \Y)
\xrightarrow[\gamma \to \infty]{} M^C(\X, \Y).
\end{equation}

For all $\C = C(\X, \Y) \in \Rmn$, $m \neq n$:
\begin{equation}
\dtw(\C) \xleftarrow[0 \leftarrow \gamma]{} \dgC(\X, \Y)
\xrightarrow[\gamma \to \infty]{} \infty.
\end{equation}
    
For all $\C = C(\X, \Y) \in \Rmn$:
\begin{equation}
\dtw(\C) \xleftarrow[0 \leftarrow \gamma]{} \sgC(\X, \Y)
\xrightarrow[\gamma \to \infty]{} M^C(\X, \Y).
\end{equation}
\label{prop:limits}
\end{proposition}
Note that the mean-cost divergence was obtained mostly as a side product of our
limit case analysis. As we show in our experiments, it performs worse than 
the (sharp) soft-DTW divergence in practice. 
Therefore we do not recommend it in practice.

\begin{figure*}[t]
\centering
\includegraphics[width=0.98 \textwidth]{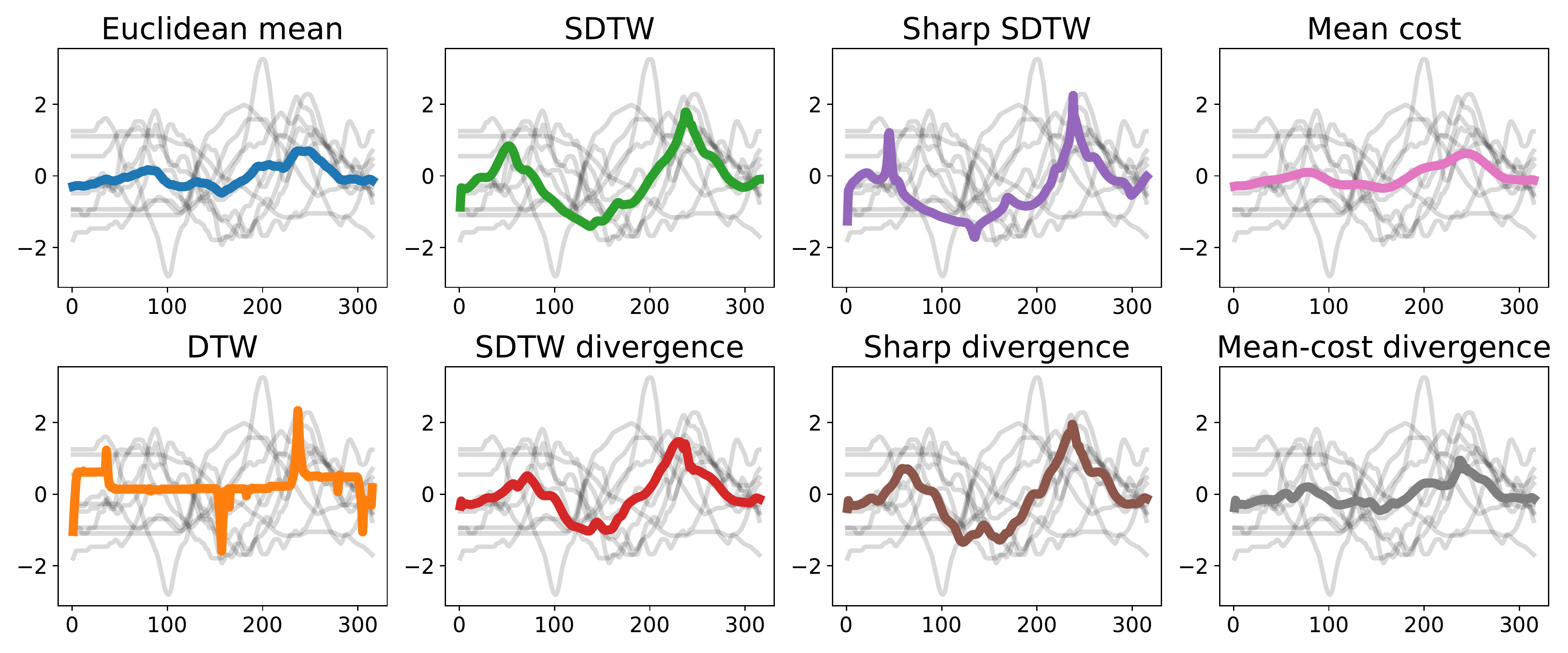}
\caption{Average of $10$ time series $\Y_1,\dots,\Y_{10}$, on the {\bf
uWaveGestureLibrary\_Y} dataset.}
\label{fig:barycenter_main}
\end{figure*}

\paragraph{Computation.}

The value, gradient, directional derivative and Hessian product of $\sdtwg(\C)$
for $\C \in \Rmn$ can all be computed in $O(mn)$ time \citep{soft_dtw,diff_dp}.
Therefore, both $\dgC(\X, \Y)$ and $\sgC(\X, \Y)$ take $O(\max\{m,n\}^2)$ time
to compute. Sharp divergences take roughly twice more time to compute, as
computing a Hessian-vector product requires one more pass through the dynamic
programming recursion.  The mean alignment and mean cost can also both be
computed in $O(mn)$ time.  We detail all algorithms in Appendix
\ref{appendix:algorithms}. 

\paragraph{Comparison with Sinkhorn divergences.}

Since our proposed divergences use similar correction terms as Sinkhorn
divergences, we briefly review them
and discuss their differences. Given two input probability measures
$\balpha \in \triangle^m$ and $\bbeta \in \triangle^n$, entropy-regularized
optimal transport is now commonly defined as
\begin{equation}
\OT_\gamma(\balpha, \bbeta) \coloneqq 
\min_{\T \in \cU(\balpha, \bbeta)} \langle \T, \C \rangle + 
\gamma \text{KL}(\T || \balpha \otimes \bbeta),
\label{eq:ot_kl}
\end{equation}
where $\text{KL}$ is the Kullback-Leibler divergence and $\cU(\balpha, \bbeta)$
is the so-called transportation polytope \citep{peyre_2019}. 
To address the entropic bias of $\OT_\gamma$, 
Sinkhorn divergences include correction terms, i.e., they are defined as
$(\balpha, \bbeta) \mapsto 
\OT_\gamma(\balpha, \bbeta) 
- \frac{1}{2} \OT_\gamma(\balpha, \balpha)
- \frac{1}{2} \OT_\gamma(\bbeta, \bbeta)$.
There are however two important differences between $\OT_\gamma$ and
$\sdtwg(C(\cdot, \cdot))$. 
First, the former is convex in its inputs (separately) while the latter is not.
This means that the proof technique for non-negativity of
Sinkhorn divergences \citep{feydy_2019} does not apply to the soft-DTW
divergence.  
Indeed our proof technique for Proposition \ref{prop:non_negativity_log} is
completely different than for Sinkhorn divergences.
Second, the entropic regularization in $\sdtwg$ is on the
probability distribution (Proposition \ref{table:properties}), not on the soft
alignment, as is the case for the transportation map $\T$ in \eqref{eq:ot_kl}.
Contrary to Sinkhorn divergences, the soft-DTW and sharp divergences are
non-convex in their inputs. For time-series averaging, an initialization scheme
that works well in practice is to use the $\sdtwg$ solution as initialization,
itself initialized from the Euclidean mean.

\section{Experimental results}
\label{sec:exp}

Throughout this section, we use the UCR (University of California, Riverside)
time series classification archive \citep{UCRArchive}. We use a subset
containing 84 datasets encompassing a wide variety of fields (astronomy,
geology, medical imaging) and lengths. Datasets include class information (up
to 60 classes) for each time series and are split into train and test sets. Due
to the large number of datasets in the UCR archive, we choose to report only a
summary of our results in the main manuscript. Detailed results are included in
the appendix for interested readers. In all experiments, we use the squared
Euclidean cost~\eqref{eq:squared_euclidean}.  
Our Python source code is available on
\href{https://github.com/google-research/soft-dtw-divergences}{github}.

\subsection{Time series averaging}

\paragraph{Experimental setup.}

To investigate the effect of our divergences on time series averaging, we
replace $\sdtwg$ in objective \eqref{eq:barycenter_obj} with our divergences.
For this task, we focus on a visual
comparison and refrain from reporting quantitative results, since the choice of
evaluation metric necessarily favors one divergence over others.
For each dataset, we pick $10$ time series $\Y_1,\dots,\Y_{10}$ randomly.
Since the time series all have the same length, we use uniform weights $w_1 =
\dots = w_k = 1$.
To approximately minimize the objective function, we use $200$ iterations of
L-BFGS \citep{lbfgs}. Because the objective is non-convex in $\X$,
initialization is important. For $\dtw$, $\sdtwg$, $\sharpg$ and $\meancost$, we
use the Euclidean mean as initialization and set $\gamma=1$. For $\dgC$, $\sgC$
and $M^C$, we use as initialization the solution of their ``biased couterpart'',
i.e., $\sdtwg$, $\sharpg$, $\meancost$, respectively, and we set $\gamma=10$. 

\begin{figure*}[t]
\centering
\includegraphics[width=0.98 \textwidth]{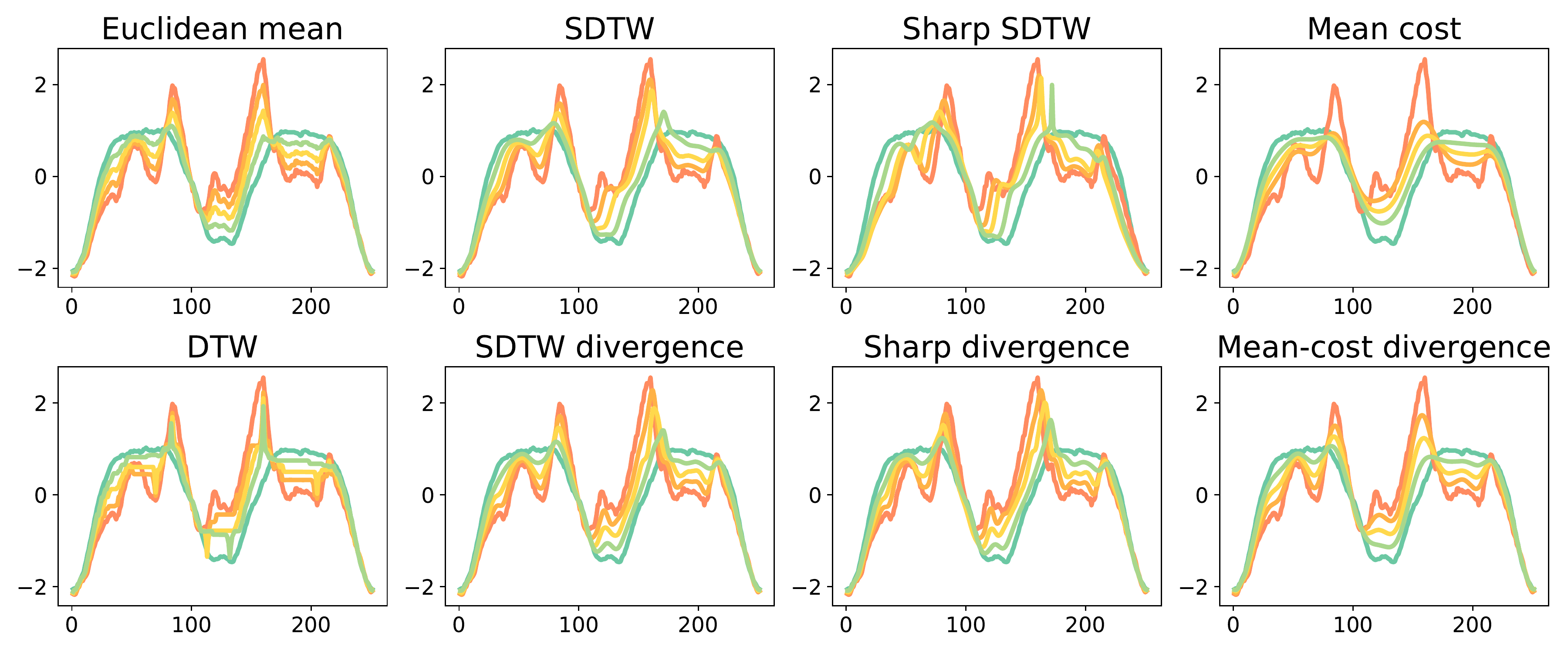}
\caption{Interpolation between two time series $\Y_1$ (red) and $\Y_2$ (dark
    green), from the {\bf ArrowHead} dataset.}
\label{fig:interpolation_main}
\end{figure*}

\paragraph{Results.}

We show the time series averages obtained on the \textit{uWaveGestureLibrary\_Y} dataset
in Figure \ref{fig:barycenter_main}.  With DTW, the obtained average does not
match well the time series, confirming the conclusion of \citet{soft_dtw}. This
is because the objective is both highly non-convex and non-smooth, rendering
optimization difficult, despite the use of Euclidean mean as initialization.  On
the other hand, the averages obtained by other divergences appear to match the
time series much better, thanks to the smoothness of their objective function.
We observe that $\dgC$ (soft-DTW divergence), $\sgC$ (sharp divergence) and
$M^C$ (mean-cost divergence) produce different results from their
biased counterpart, $\sdtwg$ (soft-DTW), $\sharpg$ (sharp soft-DTW) and
$\meancost$ (mean cost), respectively.  This is to be expected, since the
variable $\X$ with respect to which we minimize is involved in the correcting
term using $C(\X, \X)$. 
The averages obtained with $\sharpg$ and $\sgC$ tend to
include sharper peaks, a trend confirmed on other datasets as well. More
average examples are included in the appendix. 

\subsection{Time series interpolation}

\paragraph{Experimental setup.}

As a simple variation of time series averaging, we now consider time series
interpolation. We pick two times series $\Y_1$ and $\Y_2$ and set the
weights in objective \eqref{eq:barycenter_obj} to $w_1 = \pi$ and $w_2 = 1 -
\pi$, for $\pi \in \{0.25, 0.5, 0.75\}$, i.e., we seek an interpolation of the
two time series. We again minimize the objective approximately using L-BFGS,
with the same initialization scheme and the same $\gamma$ as before.

\paragraph{Results.}

Results on the \textit{ArrowHead} dataset are shown in Figure
\ref{fig:interpolation_main}. We observe similar trends as for time series
averaging. The interpolations obtained by DTW include artifacts that do not
represent well the data. Our divergences obtain slightly more visually pleasing
results than their biased counterparts. More examples are included in the
appendix. The interpolation obtained by the sharp soft-DTW includes a peak
(light green) which is slightly off, but this is not the case of the sharp
divergence.

\subsection{Time series classification}

\paragraph{Experimental setup.}

To quantitatively compare our proposed divergences, we now consider time
series classification tasks. To better isolate the effect of the divergence
itself, we choose two simple classifiers: nearest neighbor and nearest centroid.
To predict the class of a time series, the well-known nearest neighbor
classifier assigns the class of the nearest
time series in the training set, according to the chosen divergence. 
Note that this does not require differentiability of the divergence.
The lesser known nearest
centroid classifier \citep{ESLII} first computes the centroid (average) of each
class in the training set. We compute the centroid by minimizing
\eqref{eq:barycenter_obj} for each class, according to the chosen divergence.
To predict the class of a time series, we then assign the class of the nearest
centroid, according to the same divergence. 
Although very simple, this method is known to be competitive with the nearest
neighbor classifier, while requiring much lower computational cost at prediction
time \citep{petitjean_icdm}.

For all datasets in the UCR archive, we use
the pre-defined test set. For divergences including a $\gamma$ parameter, we
select $\gamma$ by cross-validation. More precisely, we train on $2/3$ of the
training set and evaluate the goodness of a $\gamma$ value on the held-out
$1/3$. We repeat this procedure $5$ times, each with a different random split,
in order to get a better estimate of the goodness of $\gamma$.
We do so for $\gamma \in \{10^{-4}, 10^{-3}, \dots, 10^{4}\}$ and select the
best one. Finally, we retrain on the entire training set using that $\gamma$
value.

\paragraph{Results.}

Due to the large number of datasets in the UCR archive, we only show a summary
of the results in Table \ref{table:summary_1nn} and Table
\ref{table:summary_nearest_centroid}. Detailed results are in Appendix
\ref{appendix:exp}. We observe consistent trends for both the
nearest neighbor and the nearest centroid classifiers.  The mean-cost divergence
appears to perform poorly, even worse than the squared Euclidean distance and
DTW.  This shows that considering all
possible alignments uniformly does not lead to a good divergence measure.  On
the other hand, our proposed divergences, the soft-DTW divergence and the sharp
divergence, outperform on the majority of the datasets the Euclidean distance,
DTW, soft-DTW, and sharp soft-DTW.  Furthermore, each proposed divergence (i.e.,
with correction term) clearly outperforms its biased counterpart (i.e., without
correction term). This shows that proper divergences, which are minimized when
the two time series are equal, indeed
translate to higher classification accuracy in practice.
Overall, the soft-DTW divergence works better than the sharp divergence.

\begin{table*}[t]
\caption{{\bf Nearest neighbor results.} Each number indicates the
percentage of datasets in the UCR archive for which using $A$ in the nearest
neighbor classifier is within $99\%$ or better than using $B$ .}
\centering
\vspace{0.3em}
\begin{tabular}{lcccccccc}
\toprule
$A$ ($\downarrow$) vs. $B$ ($\rightarrow$)
 & Euc. & DTW & SDTW & SDTW div & Sharp & Sharp div & Mean cost & Mean-cost div \\[0.3em]
\midrule
Euclidean &  -  & 41.67 & 34.62 & 22.37 & 29.49 & 27.63 & 95.29 & 71.43 \\[0.3em]
DTW & 71.43 &  -  & 42.31 & 39.47 & 50.00 & 39.47 & 89.29 & 79.76 \\[0.3em]
SDTW & 75.64 & 82.05 &  -  & 52.63 & 73.08 & 55.26 & 97.44 & 80.77 \\[0.3em]
SDTW div & 93.42 & 93.42 & 86.84 &  -  & 84.21 & 82.67 & 97.37 & 96.05 \\[0.3em]
Sharp & 83.33 & 84.62 & 76.92 & 53.95 &  -  & 52.63 & 98.72 & 87.18 \\[0.3em]
Sharp div & 94.74 & 86.84 & 77.63 & 66.67 & 81.58 &  -  & 98.68 & 96.05 \\[0.3em]
Mean cost & 9.41 & 13.10 & 8.97 & 5.26 & 5.13 & 6.58 &  -  & 44.05 \\[0.3em]
Mean-cost div & 42.86 & 32.14 & 25.64 & 19.74 & 21.79 & 18.42 & 98.81 &  -  \\[0.3em]
\hline
\end{tabular}
\label{table:summary_1nn}
\end{table*}

\begin{table*}[t]
\caption{{\bf Nearest centroid results.} Each number indicates the
percentage of datasets in the UCR archive for which using $A$ in the nearest
neighbor classifier is within $99\%$ or better than using $B$ .}
\centering
\vspace{0.3em}
\begin{tabular}{lcccccccc}
\toprule
$A$ ($\downarrow$) vs. $B$ ($\rightarrow$)
 & Euc. & DTW & SDTW & SDTW div & Sharp & Sharp div & Mean cost & Mean-cost div \\[0.3em]
\midrule
Euclidean &  -  & 44.71 & 27.06 & 28.57 & 30.95 & 32.50 & 77.65 & 78.82 \\[0.3em]
DTW & 63.53 &  -  & 36.47 & 36.90 & 41.67 & 37.50 & 83.53 & 80.00 \\[0.3em]
SDTW & 82.35 & 85.88 &  -  & 55.95 & 77.38 & 62.50 & 94.12 & 94.12 \\[0.3em]
SDTW div & 82.14 & 83.33 & 82.14 &  -  & 78.57 & 70.00 & 91.67 & 94.05 \\[0.3em]
Sharp & 79.76 & 78.57 & 54.76 & 48.81 &  -  & 55.00 & 91.67 & 91.67 \\[0.3em]
Sharp div & 82.50 & 82.50 & 70.00 & 63.75 & 78.75 &  -  & 92.50 & 93.75 \\[0.3em]
Mean cost & 37.65 & 22.35 & 11.76 & 11.90 & 15.48 & 11.25 &  -  & 77.65 \\[0.3em]
Mean-cost div & 41.18 & 23.53 & 14.12 & 14.29 & 17.86 & 15.00 & 90.59 &  -  \\[0.3em]
\hline
\end{tabular}
\label{table:summary_nearest_centroid}
\end{table*}

\section{Conclusion}

Due to entropic bias, soft-DTW can be negative and is not minimized when the two
time series are equal. To address these issues, we proposed the
soft-DTW divergence and its sharp variant. We proved that
the former is a valid divergence under the cost \eqref{eq:log_cost} for $d \in
\NN$ and under the absolute cost \eqref{eq:lap_cost} for $d=1$. 
We conjecture that this is also true under
the squared Euclidean cost \eqref{eq:squared_euclidean}, but leave a proof to
future work.
By studying the limit behavior of our divergences
when the regularization parameter $\gamma$ goes to infinity, we also obtained a
new mean-cost divergence, which is of independent interest. 
Experiments on $84$ time series classification
datasets established that the soft-DTW divergence performs the best among all
discrepancies and divergences considered.



\onecolumn
\appendix

\begin{center}
\Huge Appendix
\end{center}

\section{Algorithms}
\label{appendix:algorithms}

We begin by recalling the algorithms derived by \citet{diff_dp} for computing
the value, gradient, directional derivative and Hessian product of $\sdtwg(\C)$
in $O(mn)$ time and space. The lines in light gray indicate values that must be
set in order to handle edge cases. The Gibbs distribution \eqref{eq:proba} is
equivalent to a random walk (finite Markov chain)
on the directed acyclic graph pictured in Figure \ref{fig:dag}. The matrix $\P
\in (0,1]^{m \times n \times 3}$ computed in Algorithm \ref{alg:sdtw_value}
contains the transition probabilities for this random walk.
Although modern automatic differentiation frameworks can in principle derive
Algorithms \ref{alg:sdtw_gradient}--\ref{alg:sdtw_hess_prod} automatically from
the first output of Algorithm \ref{alg:sdtw_value}, these frameworks are
typically not well suited for tight loops operating over triplets of values,
such as the ones in Algorithm \ref{alg:sdtw_value}. We argue that a manual
implementation of the algorithms below is more efficient on CPU.
The algorithms also play an important role to compute $\sharpg(\C)$ and
$\meancost(\C)$, as we describe later.
\begin{figure}[h]
\begin{algorithm}[H]
\begin{algorithmic}
\Input Cost matrix $\C \in \Rmn$, $\gamma \ge 0$
\State \mygray{$\V_{:, 0} \leftarrow \infty$, 
$\V_{0,:} \leftarrow \infty$,
$V_{0,0} \leftarrow 0$}

\For{$i \in [1,\dots, m]$, $j \in [1,\dots,n]$}
\State $V_{i,j} \leftarrow C_{i,j} +
                        \smin_\gamma(V_{i,j-1}, V_{i-1,j-1}, V_{i-1,j})
                        \in \RR$
\State $\P_{i,j} \leftarrow 
\nabla \smin_\gamma(V_{i,j-1}, V_{i-1,j-1}, V_{i-1,j}) \in \triangle^3$
\EndFor
\State \Return $\sdtwg(\C) = V_{m,n} \in \RR$, $\P \in (0,1]^{m \times n \times
3}$
\end{algorithmic}
\caption{Soft-DTW value and transition probabilities}
\label{alg:sdtw_value}
\end{algorithm}
\vspace{-0.3cm}
\begin{algorithm}[H]
\begin{algorithmic}
\Input $\P \in (0,1]^{m \times n \times 3}$ (Algorithm \ref{alg:sdtw_value}
or Algorithm \ref{alg:cardinality})
\State \mygray{$E_{m+1,:} \leftarrow 0$, 
$E_{:, n+1} \leftarrow 0$, 
$E_{m+1, n+1} \leftarrow 1$, 
$\P_{m + 1, :} \leftarrow (0, 0, 0)$,
$\P_{:, n+1} \leftarrow (0, 0, 0)$,
$\P_{m + 1, n+1} \leftarrow (0, 1, 0)$}

\For{$j \in [n,\dots, 1]$, $i \in [m,\dots,1]$}
\State $E_{i,j} \leftarrow 
P_{i, j + 1, 1} \cdot E_{i, j + 1} +
P_{i + 1, j + 1, 2} \cdot E_{i + 1, j + 1} +
P_{i + 1, j, 3} \cdot E_{i + 1, j}$
\EndFor
\State \Return $\nabla_\C \sdtwg(\C) = \E \in (0, 1]^{m \times n}$
\end{algorithmic}
\caption{Soft-DTW gradient (expected alignment)}
\label{alg:sdtw_gradient}
\end{algorithm}
\vspace{-0.3cm}
\begin{algorithm}[H]
\begin{algorithmic}
\Input $\P \in (0,1]^{m \times n \times 3}$ (Algorithm \ref{alg:sdtw_value}
or Algorithm \ref{alg:cardinality}),
$\Z \in \Rmn$
\State \mygray{${\dot \V}_{:, 0} \leftarrow 0$, ${\dot \V}_{0,:} \leftarrow 0$}

\For{$i \in [1,\dots, m]$, $j \in [1,\dots,n]$}
\State ${\dot V}_{i,j} \leftarrow Z_{i,j} +
P_{i, j, 1} \cdot {\dot V}_{i, j - 1} + 
P_{i, j, 2} \cdot {\dot V}_{i - 1, j - 1} +
P_{i, j, 3} \cdot {\dot V}_{i - 1, j}$
\EndFor
\State \Return $\langle \nabla_\C \sdtwg(\C), \Z \rangle = 
\dot V_{m,n} \in \RR$, 
${\dot \V} \in \Rmn$
\end{algorithmic}
\caption{Soft-DTW directional derivative in the direction of $\Z$
and intermediate computations}
\label{alg:sdtw_direct_deriv}
\end{algorithm}
\vspace{-0.3cm}
\begin{algorithm}[H]
\begin{algorithmic}
\Input $\P \in (0,1]^{m \times n \times 3}$ (Algorithm \ref{alg:sdtw_value}),
${\dot \V} \in \Rmn$ (Algorithm \ref{alg:sdtw_direct_deriv}),
$\Z \in \Rmn$
\State \mygray{${\dot \E}_{m+1, :} \leftarrow 0$, 
${\dot \E}_{:,n+1} \leftarrow 0$
${\dot \P}_{m+1,:} \leftarrow (0, 0, 0)$
${\dot \P}_{:,n+1} \leftarrow (0, 0, 0)$
}

\For{$j \in [n,\dots, 1]$, $i \in [m,\dots,1]$}
\State $s \leftarrow
P_{i, j, 1} \cdot {\dot V}_{i, j - 1} +
P_{i, j, 2} \cdot {\dot V}_{i - 1, j - 1} +
P_{i, j, 3} \cdot {\dot V}_{i - 1, j}$

\State ${\dot P}_{i, j, 1} 
\leftarrow P_{i, j, 1} \cdot (s - {\dot V}_{i, j - 1})$,
${\dot P}_{i, j, 2} \leftarrow P_{i, j, 2} \cdot (s - {\dot V}_{i-1, j - 1})$,
${\dot P}_{i, j, 3} \leftarrow P_{i, j, 3} \cdot (s - {\dot V}_{i-1, j})$

\State ${\dot E}_{i,j} \leftarrow
{\dot P}_{i, j + 1, 1} \cdot E_{i, j + 1} +
P_{i, j + 1, 1} \cdot {\dot E}_{i, j + 1} +
{\dot P}_{i + 1, j + 1, 2} \cdot E_{i + 1, j + 1} + 
P_{i + 1, j + 1, 2} \cdot {\dot E}_{i + 1, j + 1} +$
\State ~~~~~~~~~~${\dot P}_{i + 1, j, 3} \cdot E_{i + 1, j} + 
P_{i + 1, j, 3} \cdot {\dot E}_{i + 1, j}$
\EndFor
\State \Return $\nabla^2_\C \sdtwg(\C) \Z = \dot \E \in \Rmn$
\end{algorithmic}
\caption{Soft-DTW Hessian product}
\label{alg:sdtw_hess_prod}
\end{algorithm}
\end{figure}

\clearpage
Since $\sharpg(\C)$ is the directional derivative of $\sdtwg(\C)$ in the
direction of $\C$, we can compute it using Algorithm \ref{alg:sdtw_direct_deriv}
with $\P$ coming from Algorithm \ref{alg:sdtw_value} and $\Z=\C$. The gradient
of $\sharpg(\C)$ w.r.t.\ $\C$, see \eqref{eq:sharp_gradient}, involves the
product with the Hessian of $\sdtwg(\C)$ and can be computed using Algorithm
\ref{alg:sdtw_hess_prod}, again with $\Z=\C$.

We continue with an algorithm to compute $\meancost(\C)$. This algorithm is new
to our knowledge. We start by a known recursion for computing the cardinality
$|\cA(m,n)|$ \citep{sulanke_2003}. The key modification we make is to build a
transition probability matrix $\P$ along the way, mirroring Algorithm
\ref{alg:sdtw_value}.
\begin{figure}[h]
\begin{algorithm}[H]
\begin{algorithmic}
\Input Cost matrix $\C \in \Rmn$
\State \mygray{$\V_{:, 0} \leftarrow 0$, 
$\V_{0,:} \leftarrow 0$,
$V_{0,0} \leftarrow 1$}

\For{$i \in [1,\dots, m]$, $j \in [1,\dots,n]$}
\State $V_{i,j} \leftarrow V_{i,j-1} + V_{i-1,j-1} + V_{i-1,j}$
\State $P_{i, j, 1} \leftarrow V_{i, j - 1} / V_{i, j}$,
$P_{i, j, 2} \leftarrow V_{i - 1, j - 1} / V_{i, j}$,
$P_{i, j, 3} \leftarrow V_{i - 1, j} / V_{i, j}$.
\EndFor
\State \Return $|\cA(m,n)| = V_{m,n} \in \mathbb{N}$, 
$\P \in (0,1]^{m \times n \times 3}$
\end{algorithmic}
\caption{Cardinality $|\cA(m,n)|$ and transition probabilities}
\label{alg:cardinality}
\end{algorithm}
\end{figure}

This modification allows us to reuse previous algorithms. 
Indeed, we can now compute $\meancost(\C)$ by using Algorithm
\ref{alg:sdtw_direct_deriv} with the above $\P$ and $\Z=\C$ as inputs.
Alternatively, we can use Algorithm \ref{alg:sdtw_gradient} to compute $\E =
\mathbb{E}[A]$, where $A$ is uniformly distributed over $\cA(m,n)$, to then
obtain $\meancost(\C) = \langle \E, \C \rangle$. Note that $\E$ is also the
gradient of $\meancost(\C)$ w.r.t. $\C$.

To summarize, we have described algorithms for computing $\sdtwg(\C)$,
$\sharpg(\C)$ and $\meancost(\C)$ in $O(mn)$ time and space. These, in turn, can
be used to compute $\dgC(\X, \Y)$ (soft-DTW divergence), $\sgC(\X,\Y)$ (sharp
divergence) and $M^C(\X, \Y)$ (mean-cost divergence)
in $O(\max\{m,n\}^2)$ time.

\section{Proofs}

\subsection{Sensitivity analysis w.r.t.\ $\gamma$}

\begin{proposition}{Derivatives w.r.t.\ $\gamma$}

We have for all $\C \in \Rmn$
\begin{equation}
\partialfrac{\sdtwg(\C)}{\gamma} 
= -H(\p_\gamma(\C)) \le 0 
\quad \text{and} \quad
\frac{\partial^2 \sdtwg(\C)}{\partial \gamma^2} =
\frac{1}{\gamma^3} \langle \C,\nabla^2_\C \sdtwg(\C) \C \rangle \le 0.
\end{equation}
\label{prop:derivatives_gamma}
\end{proposition}
\begin{proof}
Recalling that $\sdtwg(\C) = \gamma \sdtw_1(\C / \gamma)$, we have
\begin{align}
\partialfrac{\sdtwg(\C)}{\gamma}
&= \sdtw_1(\C / \gamma)
- \frac{1}{\gamma} \langle \E_1(\C / \gamma), \C \rangle \\
&= \frac{1}{\gamma} \sdtwg(\C) 
- \frac{1}{\gamma} \langle \E_\gamma(\C), \C \rangle \\
&= -H(\p_\gamma(\C)) \le 0,
\end{align}
where we used \eqref{eq:decomposition} and the fact that
$H$ is non-negative over the simplex.
Similarly, we have
\begin{align}
\frac{\partial^2 \sdtwg(\C)}{\partial \gamma^2} 
&= -\frac{1}{\gamma^2} \langle \E_1(\C / \gamma), \C \rangle
+\frac{1}{\gamma^2} \langle \E_1(\C / \gamma), \C \rangle
+ \frac{1}{\gamma^3} \langle \C, \nabla^2_\C \sdtw_1(\C / \gamma) \C \rangle \\
&= \frac{1}{\gamma^3} \langle \C,\nabla^2_\C \sdtwg(\C) \C \rangle \le 0,
\end{align}
where we used the concavity of $\sdtwg$ w.r.t.\ $\C$.
\end{proof}

\subsection{Product with the Jacobian of the squared Euclidean cost}

For the squared Euclidean cost \eqref{eq:squared_euclidean}, we have
\begin{equation}
C(\X, \Y) = 
\frac{1}{2} \diag(\X \X^\top) \ones_n^\top +
\frac{1}{2} \ones_m \diag(\Y \Y^\top)^\top
- \X \Y^\top \in \RR^{m \times n}
\end{equation}
where $\diag(\M)$ is a vector containing the diagonal elements of $\M$.
With some abuse of notation, we denote
\begin{equation}
C(\X) \coloneqq C(\X, \X) \in \Rmm.
\end{equation}

\paragraph{Product with the Jacobian transpose (``VJP'').}

For fixed $\Y \in \Rnd$, we have for all $\E \in \Rmn$
\begin{equation}
[(J_\X C(\X, \Y))^\top \E]_{i,k} 
= \sum_{j=1}^n e_{i,j} (x_{i, k} - y_{j,k})
\quad i \in [m], k \in [d]
\label{eq:sqe_vjp_sum}
\end{equation}
or equivalently
\begin{equation}
(J_\X C(\X, \Y))^\top \E = \X \circ (\E \ones_{n \times d}) - \E \Y \in \Rmd,
\label{eq:sqe_vjp}
\end{equation}
where $\circ$ denotes the Hadamard product. Similarly, we have
for all $\E \in \Rmm$
\begin{equation}
[(J_\X C(\X))^\top \E]_{i,k} = 
\sum_{j=1}^n (e_{i,j} + e_{j,i}) (x_{i, k} - x_{j, k})
\quad i \in [m], k \in [d]
\label{eq:sqe_vjp_self_sum}
\end{equation}
or equivalently
\begin{equation}
(J_\X C(\X))^\top \E = 
\X \circ ((\E + \E^\top) \ones_{m \times d}) - (\E + \E^\top) \X
\in \Rmd.
\label{eq:sqe_vjp_self}
\end{equation}
If $\E$ is symmetric, we therefore have at $\X=\Y$
\begin{equation}
(J_\X C(\X))^\top \E = 2 (J_\X C(\X, \Y))^\top \E.
\label{eq:vjp_symmetric}
\end{equation}

\paragraph{Product with the Jacobian (``JVP'').}

For fixed $\Y$, we have for all $\Z \in \Rmd$
\begin{equation}
[J_\X C(\X, \Y) \Z]_{i,j} 
= \sum_{k=1}^d z_{i,k} (x_{i, k} - y_{j,k})
\quad i \in [m], j \in [n]
\end{equation}
or equivalently
\begin{equation}
J_\X C(\X, \Y) \Z = \diag(\X \Z^\top) \ones_n^\top - \Z \Y^\top \in \Rmn.
\end{equation}
Similarly, we have for all $\Z \in \Rmd$
\begin{equation}
[J_\X C(\X) \Z]_{i,j} =
\sum_{k=1}^d (z_{i,k} - z_{j, k}) (x_{i,k} - x_{j,k})
\quad i \in [m], j \in [m]
\end{equation}
or equivalently
\begin{equation}
J_\X C(\X) \Z = 
\diag(\X \Z^\top) \ones_m^\top + \ones_m \diag(\Z \X^\top)^\top
- \Z \X^\top - \X \Z^\top \in \Rmm.
\end{equation}
We therefore have at $\X=\Y$
\begin{equation}
J_\X C(\X) \Z = J_\X C(\X, \Y) \Z + (J_\X C(\X, \Y) \Z)^\top,
\label{eq:jvp_symmetrization}
\end{equation}
i.e., 
$J_\X C(\X) \Z$ is the symmetrization of $J_\X C(\X, \Y) \Z$.

\subsection{Proof of Proposition \ref{prop:limitations} (limitations of
$\sdtwg$)}
\label{appendix:proof_limitations}

We assume assumptions A.1-A.3 hold.

{\bf 1.} The fact that $\sdtwg(\C) \xrightarrow[\gamma \to \infty]{} -\infty$
follows from \eqref{eq:decomposition}. From Proposition
\ref{prop:derivatives_gamma},
for all $\C \in \RR^{m \times n}$, $\sdtwg(\C)$ is concave 
w.r.t.\ $\gamma$ and non-increasing on $[0, \infty)$.
Since $\dtw(\C) \ge 0$ and $\sdtwg(\C) \xrightarrow[\gamma \to \infty]{}
-\infty$, from the intermediate value theorem, there exists $\gamma_0 \in [0,
\infty)$ such that $\sdtw_\gamma(\C) \le 0$ for all $\gamma \ge \gamma_0$.

{\bf 2.} If the cost $C$ satisfies assumption A.2, 
then for any $\X \in \Rmd$ the diagonal alignment $I_m \in \cA(m,m)$ satisfies 
$\langle I_m, C(\X,\X) \rangle = \sum_{i=1}^m [C(\X,\X)]_{i,i} = 0$. Therefore,
$\dtw(C(\X, \X)) = 0$. Using the fact that $\gamma \mapsto \sdtwg(\C)$ is
non-increasing on $\gamma \in [0, \infty)$, we obtain $\sdtwg(C(\X,\X) \le 0$
for all $\gamma \in [0, \infty)$.

{\bf 3.} If the minimum of $\sdtwg(C(\X, \Y))$ is achieved at $\X = \Y$, then
the gradient \eqref{eq:sdtw_grad_X} should be equal to $\zeros_{m \times d}$ or
put differently, $\E_\gamma(C(\X, \Y))$ should be in the nullspace of $(J_\X
C(\X, \Y))^\top$. For the squared Euclidean cost, from \eqref{eq:sqe_vjp_sum}, a
matrix $\E \in \Rmn$ is in the nullspace of $(J_\X C(\X, \Y))^\top$ if for all
$i \in [m], k \in [d]$
\begin{equation}
\sum_{j=1}^n e_{i,j} (x_{i, k} - y_{j,k}) = 0.
\end{equation}
Since $e_{i,j} > 0$, this is equivalent to
\begin{equation}
x_{i,k} = \frac{\sum_{j=1}^n e_{i,j} y_{j,k}}{\sum_{j=1}^n e_{i,j}} \neq
y_{i,k}.
\end{equation}

\subsection{Proof of Proposition \ref{prop:non_negativity_log} (valid divergence)}
\label{appendix:proof_nn_log_cost}

\paragraph{Positivity with the log-augmented squared Euclidean cost.}

The fact that \eqref{eq:ga_kernel} is positive definite (p.d.) under the cost
\eqref{eq:log_cost} was proved by \citet{cuturi_2007}. More precisely, in their
Theorem 1, the authors show that the kernel $K_\gamma^C(\X, \Y) =
\exp(-\sdtw_1(\X, \Y) / \gamma)$ is positive definite if the kernel 
$k(\x, \y) \coloneqq \exp(-c(\x, \y))$ is such that 
$\tilde k \coloneqq \frac{k}{1+k}$ is positive definite.  In particular, setting
\begin{equation}
k(\x, \y) 
= \frac{\frac{1}{2} \exp(-||\x - \y||^2_2 / 2)}{1 - \frac{1}{2} \exp(-||\x -
\y||^2_2 / 2)}
= \frac{\exp(-||\x - \y||^2_2 / 2)}{2 -\exp(-||\x - \y||^2_2 / 2)}
\end{equation}
ensures that $\tilde k (\x, \y)= \frac{1}{2} \exp(-||\x - \y||^2_2 / 2)$ is
positive definite, and therefore so is $K_\gamma^C$.  The
associated cost is then, for all $\x, \y \in \RR^d$,
\begin{equation}
c(\x, \y) = - \log(k(\x, \y))
= \frac{||\x - \y||^2_2}{2} + \log\left(2 - \exp\left(-\frac{||\x -
\y||^2_2}{2}\right)\right),
\end{equation}
which is exactly the cost \eqref{eq:log_cost}.  
Using this cost, the fact that the kernel $K_\gamma^C$ is positive definite
implies that the Gram matrix 
\begin{equation}
\K =
\begin{bmatrix}
K_\gamma^C(\X, \X) & K_\gamma^C(\X, \Y) \\
K_\gamma^C(\Y, \X) & K_\gamma^C(\Y, \Y)
\end{bmatrix}
\end{equation}
is positive semi-definite (p.s.d.), i.e., its determinant
is non-negative. Using \eqref{eq:ga_kernel}, we obtain using the cost \eqref{eq:log_cost}
\begin{equation}
\text{det}(\K) 
= K_\gamma^C(\X, \X) K_\gamma^C(\Y, \Y) - K_\gamma^C(\X, \Y)^2
\ge 0
\Leftrightarrow
\dgC(\X, \Y) \ge 0\,,
\end{equation}
which proves the non-negativity of $\dgC$. We are now going to prove the
converse, i.e., the fact that if $\dgC(\X, \Y) = 0$ then $\X = \Y$. First
notice from the previous equation that if $\dgC(\X, \Y) = 0$ then
$\text{det}(\K)=0$, i.e., $\K$ is of rank at most $1$ ($\K$ is a $2
\times 2$ matrix).  \citet{cuturi_2007}
showed that when $\tilde{k}$ is a positive definite kernel, then
\begin{equation}
\K = \sum_{i=1}^\infty \K_i\,,
\label{eq:Ksum}
\end{equation}
where, for any $i\geq 1$, $\K_i$ is the p.s.d. Gram matrix of the positive
definite kernel $K_i$ given by:
$$
K_i(\X,\Y) = \sum_{\A\in\tilde{\cA}(i,n)}  
\sum_{\B\in\tilde{\cA}(i,m)} \prod_{j=1}^i \tilde{k}( [\A\X]_j, [\B\Y]_j)\,,
$$
where $\tilde{\cA}(u,v) \subset \cA(u,v)$ is the set of path matrices that only
use the $\downarrow$ and $\searrow$ moves. In other words, $K_i$ compares $\X$
and $\Y$ by first ``extending'' them to length $i$ by repeating some entries
(corresponding to the $i\times d$ sequences $\A\X$ and $\B\Y$), and then
comparing each of the the $i$ terms of $\A\X$ with the corresponding term in
$\B\Y$ with $\tilde{k}$. When $\X$ and $\Y$ have the same length ($m=n$), we
notice that $\tilde{\cA}(n,n)$ is reduced to the identity matrix (there is a
single way to ``extend'' $\X$ and $\Y$ to length $n$, which is not to repeat any
entry), and therefore:
$$
K_n(\X,\Y) = \prod_{j=1}^n \tilde{k}( [\X]_j, [\Y]_j)\,.
$$
This shows in particular that $K_n(\X,\X) = K_n(\Y,\Y) = \frac{1}{2^n}$ and
$K_n(\X,\Y) < \frac{1}{2^n}$ if and only if $\X\neq\Y$ (because
$\tilde{k}(\x,\y) <1/2$ if and only if $\x\neq\y$). In particular, $\K_n$ has
rank $2$ if and only if $\X\neq\Y$. Since by \eqref{eq:Ksum} $\text{rank}(\K)
\geq \max_i \text{rank}(\K_i)$, this shows that $\dgC(\X, \Y) = 0 \implies
\text{rank}(\K) < 2 \implies \text{rank}(\K_n) <2 \implies \X=\Y$. When $\X$ and
$\Y$ do not have the same length, on the other hand (assuming without loss of
generality $m<n$), then  $\tilde{\cA}(m,n) = \emptyset$ which gives $K_m(\X,\X)
= \frac{1}{2^m}$ and $K_m(\X,\Y) = K_m(\Y,\Y)=0$, i.e.,
$$
\K_m =
\begin{bmatrix}
1/2^m & 0 \\
0 & 0
\end{bmatrix}\,,
$$
showing that $\text{rank}(\K_m)=1$ and $\text{ker}(\K_m) = \text{span}\left\{
(0,1)^\top \right\}$.
Similarly,
$$
\K_n =
\begin{bmatrix}
>0 & >0 \\
>0 & 1/2^n
\end{bmatrix}\,,
$$
showing that $\K_n \times (0,1)^\top \neq 0$ and therefore $\text{ker}(\K_m)
\cap \text{ker}(\K_n) = \{0\}$. By \eqref{eq:Ksum}, $\text{ker}(\K) \subset
\text{ker}(\K_m) \cap \text{ker}(\K_n)$, and therefore $\text{ker}(\K) =\{0\}$.
In other words, when $\X$ and $\Y$ do not have the same length (which implies in
particular that $\X\neq\Y$), then $\text{det}(\K)>0$ and therefore $\dgC(\X, \Y)
> 0$. This finishes to prove that $\dgC(\X, \Y) = 0$ if and only if $\X=\Y$.

\newcommand{\cF}{\mathcal{F}}

\paragraph{Positivity with absolute value cost.}\label{app:laplacian}

We now consider the absolute value on $\RR \times \RR$
\begin{equation}\label{eq:euclidean}
c(x, y) = | x - y |,
\end{equation}
and show that $K_\gamma^C$ is positive definite for this cost.
The corresponding kernel is
\begin{equation}
k(x, y) = \exp(-c(x, y)) = \exp(-|x -y|),
\end{equation}
namely the Laplacian kernel. Following the paragraph above, we
show that $\tilde k = \frac{k}{1 + k}$ is p.d. We first note that
$\tilde k$ is translation invariant and rewrites $\tilde k(x, y) = f(x-y)$,
where
\begin{equation}
f(w) \coloneqq \frac{1}{1 + \exp(|w|)}.
\end{equation}

From Bochner's theorem, the function $f: \RR \to \RR$ is p.d. (i.e. $\tilde k$
is p.d.) if and only if it is the Fourier transform of a positive measure. Since
$f$ is integrable and square integrable, it suffices to study the sign of its
Fourier transform. For all $\omega \in
\RR$,
\begin{align}
    \cF[f](\omega) \coloneqq \int_{-\infty}^{\infty} \frac{e^{- i \omega x}}{1 + e^{|x|}} d x
    &=
    \int_{-\infty}^{0} \frac{e^{- i \omega x}}{1 + e^{-x}} d x
     +
     \int_{0}^{\infty} \frac{e^{- i \omega x}}{1 + e^{x}} d x
      \\
      &=
      \int_{0}^{\infty} \frac{e^{- i \omega x}}{1 + e^{x}} d x
       +
       \int_{0}^{\infty} \frac{e^{i \omega x}}{1 + e^{x}} d x
    \\
    &= 2 \int_0^\infty \frac{\cos(\omega x)}{1 + e^x} dx \\
    &= \frac{2}{\omega} \int_0^\infty \frac{\cos(x)}{1 + e^{x/\omega}} dx \\
    &= \frac{2}{\omega} \sum_{k=0}^\infty \int_0^{2 \pi} \frac{\cos(x)}{1 + e^{x/\omega + 2 k \pi/\omega}} dx \\
    &\coloneqq \frac{2}{\omega} \sum_{k=0}^\infty \int_0^{2 \pi} a_k.
\end{align}
Let us further decompose the sequence $(a_k)_{k=0}^\infty$ by splitting the
integral into four parts and using the periodicity of the cosine function. For
all $k \geq 0$,
\begin{equation}
     a_k =  \int_0^{\frac{\pi}{2}} \cos(x)
     \Big(\sigma_k(x)
      +\sigma_k(2 \pi - x)
      -\sigma_k(\pi + x)
      -\sigma_k(\pi - x)
     \Big) d x \coloneqq \int_0^{\frac{\pi}{2}} \cos(x) f_k(x) d x 
\end{equation}
where $\sigma_k(x) \coloneqq \frac{1}{1 + e^{\frac{2k\pi + x}{\omega}}}$.  Note
that $\sigma_k$ is convex, so that its derivative $\sigma_k'$ is increasing on
$\RR$. Therefore, for all $x \in [0, \frac{\pi}{2}]$, we have $\sigma'_k(x) \leq
\sigma_k'(\pi - x)$ and $\sigma'_k(\pi + x) \leq \sigma'_k(2 \pi - x)$. Hence,
for all $x \in [0, \frac{\pi}{2}]$,
$f'_k(x) \leq 0$, which implies $f_k(x) \geq f_k(\frac{\pi}{2}) = 0$.  We
conclude that $\cF[f] \geq 0$ on $\RR$, and therefore $\tilde k = \frac{k}{1 +
k}$ is p.d. Theorem 1 of \citet{cuturi_2007} ensures that
$K_\gamma^C$ is positive definite, so that $D_\gamma^C$ is non-negative. To
prove that $D_\gamma^C(\X,\Y)=0$ if and only if $\X=\Y$, we proceed exactly as
for the log-augmented squared Euclidean cost.

\subsection{Numerical verifications for the squared Euclidean cost case} 
\label{appendix:numerical_validation}

\paragraph{Numerical evidence of the positive definiteness of $K_\gamma^C$.}
 
We conjecture that $K_\gamma^C$ is positive definite when $C$ is the squared
Euclidean cost \eqref{eq:squared_euclidean}.  This is evidenced by the following
numerical experiment. Given $M$ time series $\X_1,\dots,\X_M$, we can form the
$M \times M$ Gram matrix defined by
\begin{equation}
    [\K]_{i,j} = K_\gamma^C(\X_i, \X_j) \quad i,j \in [M].
\end{equation}
If $K_\gamma^C$ were not positive definite, the following minimization problem
\begin{equation}
\min_{\X_1,\dots,\X_M,\v} ~ \frac{1}{||\v||^2} \v^\top \K \v
\end{equation}
would give negative values. We solved this non-convex optimization problem for
different values of $M$ using L-BFGS, and could never find negative values.
The positive definiteness of $K_\gamma^C$ would imply the non-negativity of 
$\dgC$ using the squared Euclidean cost.

\paragraph{Disproving a conjecture.}

\citet{cuturi_2007} notice that the Gaussian kernel $k(\x, \y) \coloneqq
\exp(-||\x - \y||^2 / 2)$ is such that $\frac{k}{1+k}$ empirically yields
positive semidefinite Gram matrices, and leave open the question of whether
$\frac{k}{1+k}$ is indeed a p.d. kernel, which would prove that $K_\gamma^C$ is
p.d.\ as well (cf.\ Appendix \ref{appendix:proof_nn_log_cost}). We rigorously
derive a counter-example showing that this is not the case.  The kernel $\tilde
k = \frac{k}{1+k}$ is translation invariant and rewrites
\begin{equation}
\tilde k(\x, \y) = f(\x - \y)
\quad \text{where} \quad
f(\t) \coloneqq \frac{\exp(-\|\t\|^2 / 2)}{1 + \exp(-\|\t\|^2)}.
\end{equation}
From Bochner's theorem, the function $f: \RR^d \to \RR$ is p.d.\ if and only if
it is the Fourier transform of a positive measure. Since $f$ is
integrable and square integrable, it suffices to study the sign of its Fourier
transform. For that purpose, let us rewrite $f$ as a power series:
$$
\forall \t \in\RR^d: \quad
f(\t) = \frac{e^{-\frac{||\t||^2}{2}}}{1+e^{-\frac{||\t||^2}{2}}} = 
\sum_{n=1}^\infty (-1)^{n+1} e^{-\frac{n ||\t||^2}{2}}\,.
$$
The convergence is absolute since
$$
\sum_{n=1}^\infty e^{-\frac{n ||\t||^2}{2}} 
= \frac{1}{e^{\frac{||\t||^2}{2}}-1} <\infty.
$$
Moreover, this function is integrable. By the theorem of dominated
convergence, the Fourier transform of $f$,
$$
\mathcal{F}[f](\bomega) \coloneqq \int_{\RR^d} f(\x) e^{-i\bomega^\top \x}d\x\,,
$$
is equal to a converging series of Fourier transforms:
$$
\cF[f](\bomega) = \sum_{n=1}^\infty (-1)^{n+1} \cF\left[e^{-\frac{n
||\cdot||^2}{2}}\right](\bomega)\,.
$$
It is well-known that, for any $a\in\RR_+$,
$$
\mathcal{F}\left[e^{-a ||\cdot||^2}\right](\bomega) = \left(
\frac{\pi}{a}\right)^{\frac{d}{2}} e^{-\frac{||\bomega||^2}{4 a}}\,,
$$
which gives with $a = \frac{n}{2}$
\begin{equation*}
\cF[f](\bomega) = (\pi)^{\frac{d}{2}}\sum_{n=1}^\infty
\frac{(-1)^{n+1}}{n^{\frac{d}{2}}}e^{-\frac{||\bomega||^2}{2n}}. \\
\end{equation*}

\begin{figure}[h]
    \centering
    \includegraphics{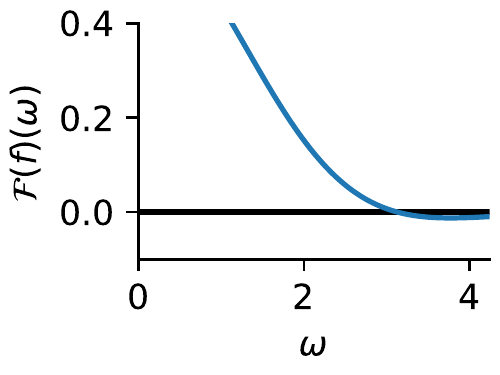}
    \caption{Fourier transform of $\tilde k = \frac{k}{1+k}$ when $k$ is the
    Gaussian kernel. The Fourier transform can be negative.}
    \label{fig:fourier}
\end{figure}

We may thus compute approximately the coefficients $\cF[f](\bomega)$ for all
$\bomega \in \RR^d$. In dimension $d=1$, truncating the series at $N = 10^6$, we
obtain the curve presented in \autoref{fig:fourier}, and observe negative
coefficients.  To ensure that the infinite sum is negative, we now bound the
residual when we truncate the sum at 2N (for $d=1$):
\begin{equation*}
\begin{split}
R_N(\bomega) 
&= \sqrt{\pi} \sum_{n=2N+1}^\infty \frac{(-1)^{n+1}}{\sqrt{n}}e^{-\frac{||\bomega||^2}{2n}} \\
&= \sqrt{\pi} \sum_{n=N}^\infty \left[ \frac{e^{-\frac{||\bomega||^2}{2(2n+1)}}}{\sqrt{2n+1}} -  \frac{e^{-\frac{||\bomega||^2}{2(2n+2)}}}{\sqrt{2n+2}} \right] \\
&\leq \sqrt{\pi} \sum_{n=N}^\infty \left[ \frac{e^{-\frac{||\bomega||^2}{2(2n+2)}}}{\sqrt{2n+1}} -  \frac{e^{-\frac{||\bomega||^2}{2(2n+2)}}}{\sqrt{2n+2}} \right] \\
&\leq \sqrt{\pi} \sum_{n=N}^\infty \left[ \frac{1}{\sqrt{2n+1}} -  \frac{1}{\sqrt{2n+2}} \right] \\
&= \sqrt{\pi} \sum_{n=N}^\infty  \frac{1}{\sqrt{2n+1}} \left[ 1 -  \sqrt{ 1 -  \frac{1}{{2n+2}} } \right] \\
&\leq \sqrt{\pi} \sum_{n=N}^\infty  \frac{1}{\sqrt{2n+1}(2n+2)} \\
&\leq  \sqrt{\frac{\pi}{8}}\sum_{n=N}^\infty  \frac{1}{n\sqrt{n}} \\
&\leq  \sqrt{\frac{\pi}{8}} \int_{N-1}^\infty \frac{dx}{x\sqrt{x}} \\
&=  \sqrt{\frac{\pi}{2(N-1)}}.
\end{split}
\end{equation*}
For $N=10^6$, this gives $R_N(\bomega) < 2\times 10^{-3}$. We observed
numerically some values strictly smaller than $-2\times 10^{-3}$ for the
truncation at  $N=10^6$ of the series: in particular, $\cF[f](2.65) = -.012$,
which implies that the infinite sum is negative.
We therefore conclude that $\frac{k}{k+1}$ is not positive definite when $k$ is
the Gaussian kernel. Note, however, that this does not disprove the positive
definiteness of $K_\gamma^C$ using the squared Euclidean cost.

\subsection{Proof of Proposition \ref{prop:stationary_point} (stationary point
using the squared Euclidean cost)}
\label{appendix:proof_stationary_point}

\paragraph{Soft-DTW divergence.}

We recall that we denote $C(\X) \coloneqq C(\X, \X) \in \Rmm$.
Using \eqref{eq:sdtw_grad_X}, we have
\begin{equation}
\nabla_\X \dgC(\X, \Y) 
= (J_\X C(\X, \Y))^\top \E_\gamma(C(\X, \Y))
- \frac{1}{2} (J_\X C(\X))^\top \E_\gamma(C(\X)).
\end{equation}
Under the squared Euclidean cost, $C(\X)$ is a symmetric matrix.
For any $\A \in \cA(m,m)$, there exists $\A^\top \in \cA(m, m)$.
Moreover for any symmetric matrix $\C$, the probability $\PP_\gamma(\A; \C)$
is the same as $\PP_\gamma(\A^\top; \C)$. 
From \eqref{eq:expectation}, we therefore have that $\E_\gamma(C(\X)) \in \Rmm$
is a symmetric matrix. In order to have
$\nabla_\X \dgC(\X, \Y) = \zeros_{m \times d}$ at $\X = \Y$, it suffices that
$(J_\X C(\X, \Y))^\top$ and $\frac{1}{2} (J_\X C(\X))^\top$ map symmetric
matrices to the same matrix.
From \eqref{eq:vjp_symmetric}, this is indeed the case for the squared
Euclidean cost.

\paragraph{Sharp divergence.}

Using \eqref{eq:sharp_gradient}, we get
\begin{align}
\nabla_\X \sharpg(C(\X,\Y)) 
&= (J_\X C(\X, \Y))^\top \nabla_\C \sharpg(C(\X, \Y)) \\
&= (J_\X C(\X, \Y))^\top 
[\E_\gamma(\C) + \frac{1}{\gamma} \nabla^2_\C \sdtwg(C(\X,\Y)) C(\X,\Y) ] \\
&= \nabla_\X \sdtwg(C(\X,\Y)) + \frac{1}{\gamma}
(J_\X C(\X, \Y))^\top \nabla^2_\C \sdtwg(C(\X,\Y)) C(\X,\Y).
\end{align}
We therefore have
\begin{align}
\nabla_\X \sgC(\X, \Y) = \nabla_\X \dgC(\X, \Y) 
&+ \frac{1}{\gamma} (J_\X C(\X, \Y))^\top \nabla^2_\C \sdtwg(C(\X,\Y)) C(\X,\Y)
\\
&- \frac{1}{2\gamma} (J_\X C(\X))^\top \nabla^2_\C \sdtwg(C(\X)) C(\X).
\label{eq:sharp_div_grad_X}
\end{align}
From the previous paragraph, we know that $\nabla_\X \dgC(\X,\Y) = \zeros_{m
\times d}$ at $\X=\Y$ using the squared Euclidean cost.
It remains to show that the sum of the other two terms in
\eqref{eq:sharp_div_grad_X} is also equal to $\zeros_{m \times d}$.
Since $(J_\X C(\X, \Y))^\top$ and $\frac{1}{2} (J_\X C(\X))^\top$ map symmetric
matrices to the same matrix using the squared Euclidean cost,
it suffices to show that $\nabla^2_\C \sdtwg(C(\X)) C(\X)$ is a symmetric
matrix.

It is well-known that the Hessian of the log-partition under a Gibbs
distribution is equal to the covariance matrix \citep{wainwright_2008}.
The Hessian can be seen as a $mn \times mn$ matrix.
Accounting for the negative sign
in \eqref{eq:sdtw}, we have
\begin{align}
\nabla_\C^2 \sdtwg(\C)
&=  -\EE_\gamma[\vect(A - \E_\gamma(\C)) \vect(A - \E_\gamma(\C))^\top] \\
&= -\sum_{\A \in \cA(m,n)} \PP_\gamma(\A; \C) \vect(\A - \E(\C)) 
\vect(\A - \E(\C))^\top \\
&=  \EE_\gamma[\vect(A)]\EE_\gamma[\vect(A)]^\top - 
\EE_\gamma[\vect(A) \vect(A)^\top],
\end{align}
where $A$ is a random alignment matrix distributed according to
$\PP_\gamma(\A;\C)$.
Equivalently, we can see the Hessian as linear map from $\Rmn$ to
$\Rmn$. Applying that map to a matrix $\M \in \Rmn$, we obtain
\begin{align}
\nabla^2_\C \sdtwg(\C) \M 
&= -\sum_{\A \in \cA(m,n)} \PP_\gamma(\A; \C) (\A - \E_\gamma(\C)) 
\langle \A - \E_\gamma(\C), \M \rangle \\
&= \langle \E_\gamma(\C), \M \rangle \E_\gamma(\C) 
- \sum_{\A \in \cA(m,n)} \PP_\gamma(\A; \C) \langle \A, \M \rangle \A 
\\
&= \langle \E_\gamma(\C), \M \rangle \E_\gamma(\C) 
- \EE_\gamma[\langle A, \M \rangle A].
\end{align}
We now assume $\C = \M = C(\X)$.
We already proved that $\E_\gamma(\C)$ is a symmetric matrix.
Using the same argument $\EE_\gamma[\langle A, \M \rangle A]$ is also symmetric.
Therefore $\nabla^2_\C \sdtwg(\C) \M$ is a symmetric matrix, concluding the
proof.

\subsection{Multiplication with the Hessian}

For completeness, we also include a discussion on the multiplication with the
Hessian w.r.t.\ $\X$.
The product between the Hessian $\nabla^2_\X \sdtwg(C(\X, \Y))$ and any $\Z \in
\Rmd$ is equal to the product between the Jacobian of $\nabla_\X \sdtwg(C(\X,
\Y))$ and $\Z$:
\begin{equation}
\nabla^2_\X \sdtwg(C(\X, \Y)) \Z
= J_\X [\nabla_\X \sdtwg(C(\X, \Y))] \Z
= J_\X [J_\X C(\X, \Y)^\top \E_\gamma(C(\X, \Y))] \Z.
\end{equation}
Using the product rule and the chain rule, we obtain
\begin{equation}
\nabla^2_\X \sdtwg(C(\X, \Y)) \Z
= 
\underbrace{[J_\X (J_\X C(\X, \Y))^\top \E_\gamma(C(\X, \Y))]}_{\B_\gamma(\X, \Y)} \Z
+ 
(J_\X C(\X, \Y))^\top
\nabla^2_\C \sdtwg(C(\X, \Y))
J_\X C(\X, \Y) \Z.
\end{equation}
Similarly,
\begin{equation}
\nabla^2_\X \sdtwg(C(\X)) \Z
= 
\underbrace{[J_\X (J_\X C(\X))^\top \E_\gamma(C(\X))]}_{\B_\gamma(\X)} \Z
+ 
(J_\X C(\X))^\top \nabla^2_\C \sdtwg(C(\X)) J_\X C(\X) \Z.
\end{equation}
From now on, we assume the squared Euclidean cost. 
Using \eqref{eq:sqe_vjp_sum}, we obtain
\begin{equation}
[\B_\gamma(\X, \Y) \Z]_{i,k} = \sum_{j=1}^n [\E_\gamma(C(\X, \Y))]_{i,j} z_{i, k}
\quad i \in [m], k \in [d]
\end{equation}
or equivalently
\begin{equation}
\B_\gamma(\X, \Y) \Z =
\Z \circ (\E_\gamma(C(\X,\Y)) \ones_{n \times d}) \in \Rmd.
\end{equation}
Similarly, using \eqref{eq:sqe_vjp_self_sum}
and the fact that $\E_\gamma(C(\X))$ is a symmetric matrix, we obtain
\begin{equation}
[\B_\gamma(\X) \Z]_{i,k} =
2 \sum_{j=1}^n [\E_\gamma(C(\X))]_{i,j} (z_{i, k} - z_{j, k})
\end{equation}
or equivalently
\begin{equation}
\B_\gamma(\X) \Z =
2 \Z \circ (\E_\gamma(C(\X) \ones_{m \times d}) - 2 \E_\gamma(C(\X)) \Z \in \Rmd.
\end{equation}
At $\X = \Y$, we therefore get
\begin{equation}
\B_\gamma(\X, \Y) \Z - \frac{1}{2} \B_\gamma(\X) \Z 
= \E_\gamma(C(\X))^\top \Z = \E_\gamma(C(\X)) \Z.
\end{equation}
At $\X = \Y$, from \eqref{eq:vjp_symmetric} and \eqref{eq:jvp_symmetrization},
we also have
\begin{equation}
(J_\X C(\X))^\top \nabla^2_\C \E_\gamma(C(\X)) J_\X C(\X) \Z
= 2 J_\X C(\X, \X)^\top 
\nabla^2_\C \sdtwg(C(\X)) (J_\X C(\X, \X) \Z + (J_\X C(\X, \X) \Z)^\top).
\end{equation}
Putting everything together, at $\X = \Y$, we have
\begin{align}
\nabla^2_\X \dgC(\X, \Y) \Z 
&= \nabla^2_\X \sdtwg(C(\X, \Y)) \Z 
- \frac{1}{2} \nabla^2_\X \sdtwg(C(\X)) \Z \\
&= \E_\gamma(C(\X)) \Z - 
J_\X C(\X, \X)^\top \nabla^2_\C \sdtwg(C(\X)) (J_\X C(\X, \X) \Z)^\top.
\end{align}
An open question is to prove that $\X=\Y$ is a local minimum, i.e.,
$\langle Z, \nabla^2_\X \dgC(\X, \Y) \Z \rangle > 0$ for all $\Z \in \Rmd$.

\subsection{Proof of Proposition \ref{prop:limits} (limits w.r.t.\ $\gamma$)}

\paragraph{Limit to zero.}

Since both $\sdtwg(\C)$ and $\sharpg(\C)$ converge to $\dtw(\C)$ when $\gamma
\to 0$, both $\dgC(\X, \Y)$ and $\sgC(\X, \Y)$ converge to
\begin{equation}
\dtw(C(\X, \Y)) - \frac{1}{2} \dtw(C(\X, \X)) - \frac{1}{2} \dtw(C(\Y, \Y)).
\end{equation}
Since the optimal alignment of $\A^\star(\C(\X, \X))$ is the identity matrix
under assumption A.2, we have $\dtw(C(\X, \X)) = 0$ and similarly 
$\dtw(C(\Y, \Y)) = 0$.
Therefore, both $\dgC(\X, \Y)$ and $\sgC(\X, \Y)$ converge to $\dtw(C(\X, \Y))$.

\paragraph{Limit to infinity.}

From \eqref{eq:varitional_form}, when $\gamma \to \infty$, the solution becomes
the maximum entropy one, $\p^\star = \ones / |\cA(m,n)|$. Hence, 
$\langle \p^\star, s(\C) \rangle$ converge to the mean cost
\eqref{eq:mean_cost}. This gives the limit for the $\sgC$ case. For the $\dgC$
case, we also need to take into account the entropy terms
\begin{equation}
-\gamma H(\p_\gamma(C(\X, \Y)) 
+ \frac{\gamma}{2} H(\p_\gamma(C(\X, \X)))
+ \frac{\gamma}{2} H(\p_\gamma(C(\Y, \Y))).
\end{equation}
When $\gamma \to \infty$, each term attains the maximum entropy value and we get
\begin{equation}
-\gamma \log |\cA(m,n)|
+ \frac{\gamma}{2} \log |\cA(m,m)|
+ \frac{\gamma}{2} \log |\cA(n,n)| = \frac{\gamma}{2} \log \frac{ |\cA(m,m)| |\cA(n,n)|}{ |\cA(m,n)|^2}\,.
\end{equation}
When $m=n$, the terms cancel out. 
Hence, $\dgC(\X, \Y)$ converge. When, $m \neq n$, the positive terms are
stronger, and the limit goes to $\infty$. By definition, we have
\begin{equation}\label{eq:limgamma}
\begin{split}
\dgC(\X, \Y) &= \sdtwg(C(\X, \Y)) -\frac{1}{2} \sdtwg(C(\X, \X)) -\frac{1}{2} \sdtwg(C(\Y, \Y)) \\
&= -\gamma \log \sum_{\A \in \cA(m,n)} \exp(-\langle \A, C(\X,\Y) \rangle /\gamma) \\
& \quad\quad+ \frac{\gamma}{2}\log \sum_{\A \in \cA(m,m)} \exp(-\langle \A, C(\X,\X) \rangle  /\gamma)+  \frac{\gamma}{2}\log \sum_{\A \in \cA(n,n)} \exp(-\langle \A, C(\Y,\Y) \rangle/
\gamma)\\
&= - \frac{\gamma}{2}\log \frac{|\cA(m,n)|^2}{|\cA(m,m)||\cA(n,n)|} -\gamma \log \left[\frac{1}{|\cA(m,n)|}\sum_{\A \in \cA(m,n)} \exp(-\langle \A, C(\X,\Y) \rangle  /\gamma)\right] \\
& \quad\quad+ \frac{\gamma}{2}\log \left[\frac{1}{|\cA(m,m)|}\sum_{\A \in \cA(m,m)} \exp(-\langle \A, C(\X,\X) \rangle /\gamma)  \right]\\
& \quad\quad+  \frac{\gamma}{2}\log \left[\frac{1}{|\cA(n,n)|}\sum_{\A \in \cA(n,n)} \exp(-\langle \A, C(\Y,\Y) \rangle  /\gamma)\right]
\end{split}
\end{equation}
Let us first consider the limit of the second term in this sum when $\gamma\rightarrow +\infty$:
\begin{equation*}
\begin{split}
\gamma \log \left[\frac{1}{|\cA(m,n)|}\sum_{\A \in \cA(m,n)} \exp(-\langle \A, C(\X,\Y) \rangle  /\gamma)\right]
&= \gamma \log \left[\frac{1}{|\cA(m,n)|}\sum_{\A \in \cA(m,n)} \left(1 - \frac{\langle \A, C(\X,\Y) \rangle}{\gamma} + o(1/\gamma)\right)\right] \\
&= \gamma \log \left[ 1 - \frac{\meancost(C(\X,\Y)) }{\gamma} + o(1/\gamma)\right] \\
&= - \meancost(C(\X,\Y)) + o(1)\,.
\end{split}
\end{equation*}
A similar computation for the third and fourth term in \eqref{eq:limgamma} leads to
\begin{equation*}
\begin{split}
\dgC(\X, \Y) 
&= - \frac{\gamma}{2}\log \frac{|\cA(m,n)|^2}{|\cA(m,m)||\cA(n,n)|} +
\meancost(C(\X,\Y)) - \frac{1}{2} \meancost(C(\X,\X)) \\
&~~~ - \frac{1}{2} \meancost(C(\Y,\Y)) + o(1) \\
&= - \frac{\gamma}{2}\log \frac{|\cA(m,n)|^2}{|\cA(m,m)||\cA(n,n)|} + M^C(\X,\Y) + o(1)\,.
\end{split}
\end{equation*}
When $m=n$, the first term is equal to $0$, so we get $\lim_{\gamma\rightarrow+\infty} \dgC(\X, \Y)  = M^C(\X,\Y)$. When $m\neq n$, on the other hand, we can use the fact that for any integers $m,n$:
$$
|\cA(m,n)| = \Delannoy(m-1,n-1)\,,
$$
where $\Delannoy(m,n)$ is the Delannoy number, i.e., the number of paths on a
rectangular grid from the origin $(0,0)$ to the northeast corner $(m,n)$, using
only single steps north, east or northeast (the $(m-1,n-1)$ term stems from the
fact that alignment matrices represent paths starting from $(1,1)$ and not
$(0,0)$). We can now use Lemma~\ref{lem:delannoy} below to get, when $m\neq n$:
$$
\log \frac{|\cA(m,n)|^2}{|\cA(m,m)||\cA(n,n)|} = \log \frac{\Delannoy(m-1,n-1)^2}{\Delannoy(m-1,m-1)\times \Delannoy(n-1,n-1)} < 0\,,
$$
and therefore that $\lim_{\gamma\rightarrow+\infty} \dgC(\X, \Y)  = + \infty$.

\begin{lemma}\label{lem:delannoy}
For any $m,n\in\mathbb{N}$, if $m\neq n$ then
$$
\log \frac{\Delannoy(m,n)^2}{\Delannoy(m,m)\times \Delannoy(n,n)} < 0\,.
$$
\end{lemma}
\begin{proof}
We use the following characterization of Delannoy numbers
\citep[e.g.,][]{banderier_2005}:
$$
\Delannoy(m,n) = \sum_{k=0}^{\min(m,n)} \left(\begin{array}{c}m \\k\end{array}\right)  \left(\begin{array}{c}n \\k\end{array}\right) 2^k\,,
$$
to obtain, assuming without loss of generality that $m<n$:
\begin{equation*}
\begin{split}
\Delannoy(m,n)^2
&=  \left[ \sum_{k=0}^{m} \left(\begin{array}{c}m \\k\end{array}\right)  \left(\begin{array}{c}n \\k\end{array}\right) 2^k \right]^2\\
&\leq  \left[ \sum_{k=0}^{m} \left(\begin{array}{c}m \\k\end{array}\right)^2 2^k \right] \times  \left[ \sum_{k=0}^{m} \left(\begin{array}{c}n \\k\end{array}\right)^2 2^k \right] \\
&< \left[ \sum_{k=0}^{m} \left(\begin{array}{c}m \\k\end{array}\right)^2 2^k \right] \times  \left[ \sum_{k=0}^{n} \left(\begin{array}{c}n \\k\end{array}\right)^2 2^k \right] \\
&= \Delannoy(m,m) \times \Delannoy(n,n) \,,
\end{split}
\end{equation*}
where we used Cauchy-Schwartz inequality for the first inequality, and the fact
that $m<n$ for the second (strict) inequality.
\end{proof}

\clearpage
\section{Additional empirical results}
\label{appendix:exp}

\begin{figure*}[h]
\centering
\includegraphics[width=0.98 \textwidth]{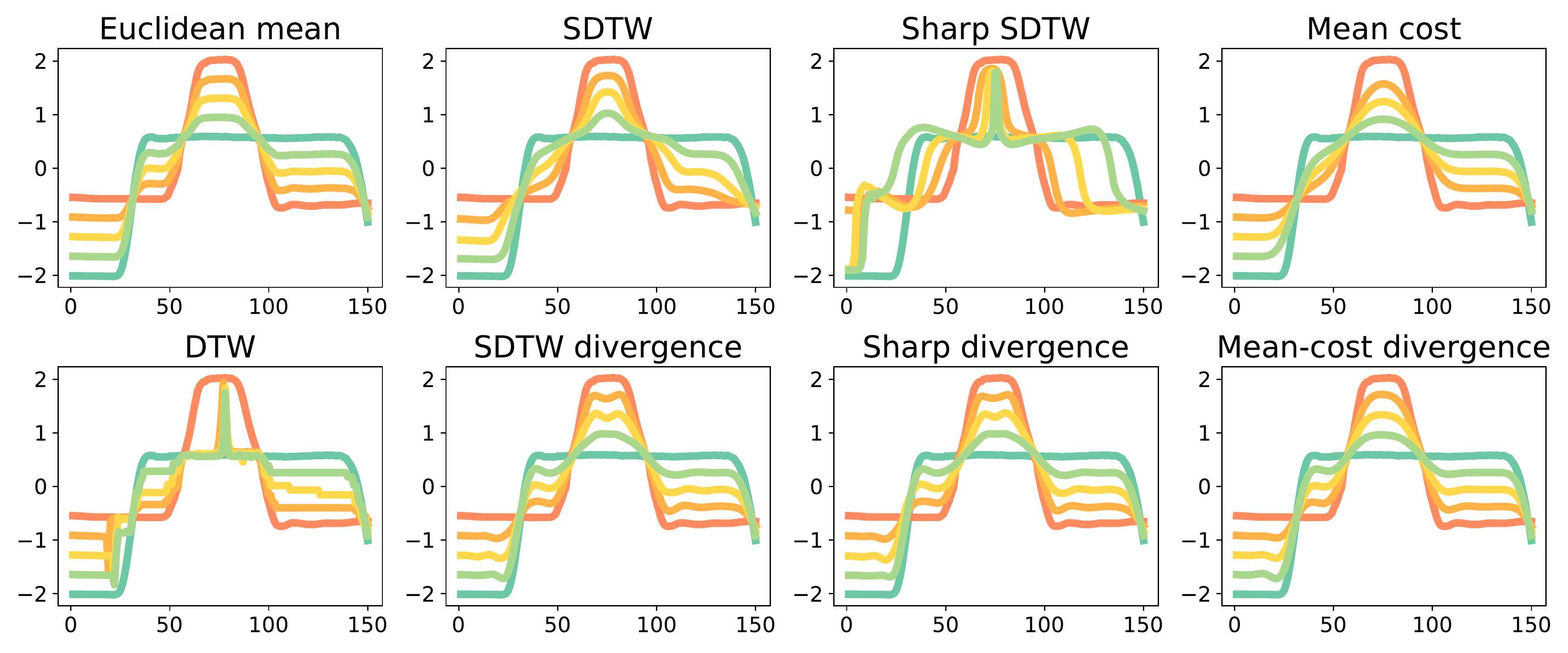}
\caption{Interpolation between two time series, from the GunPoint dataset.}
\end{figure*}

\begin{figure}[h]
\centering
\includegraphics[width=0.98 \textwidth]{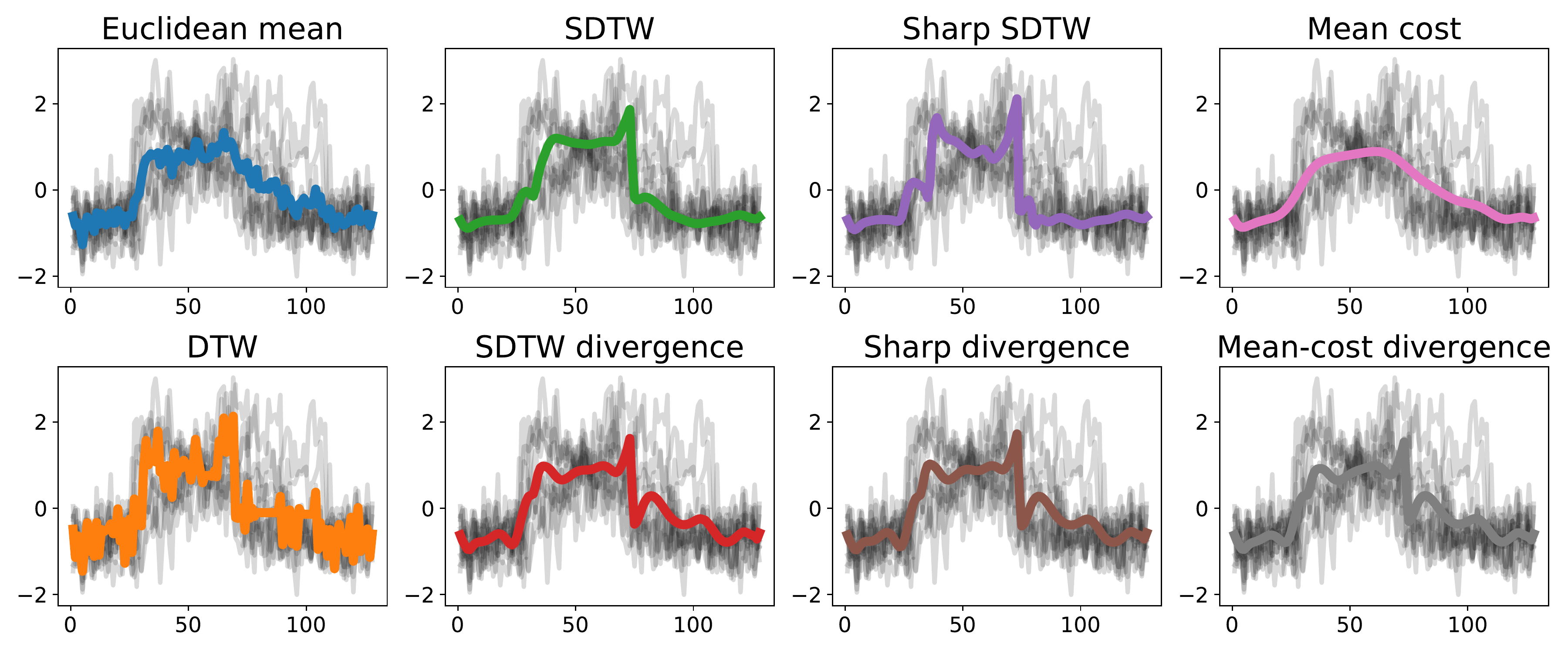}
\caption{Barycenters on the {\bf CBF} dataset.}
\end{figure}

\begin{figure}[p]
\centering
\includegraphics[width=0.95 \textwidth]{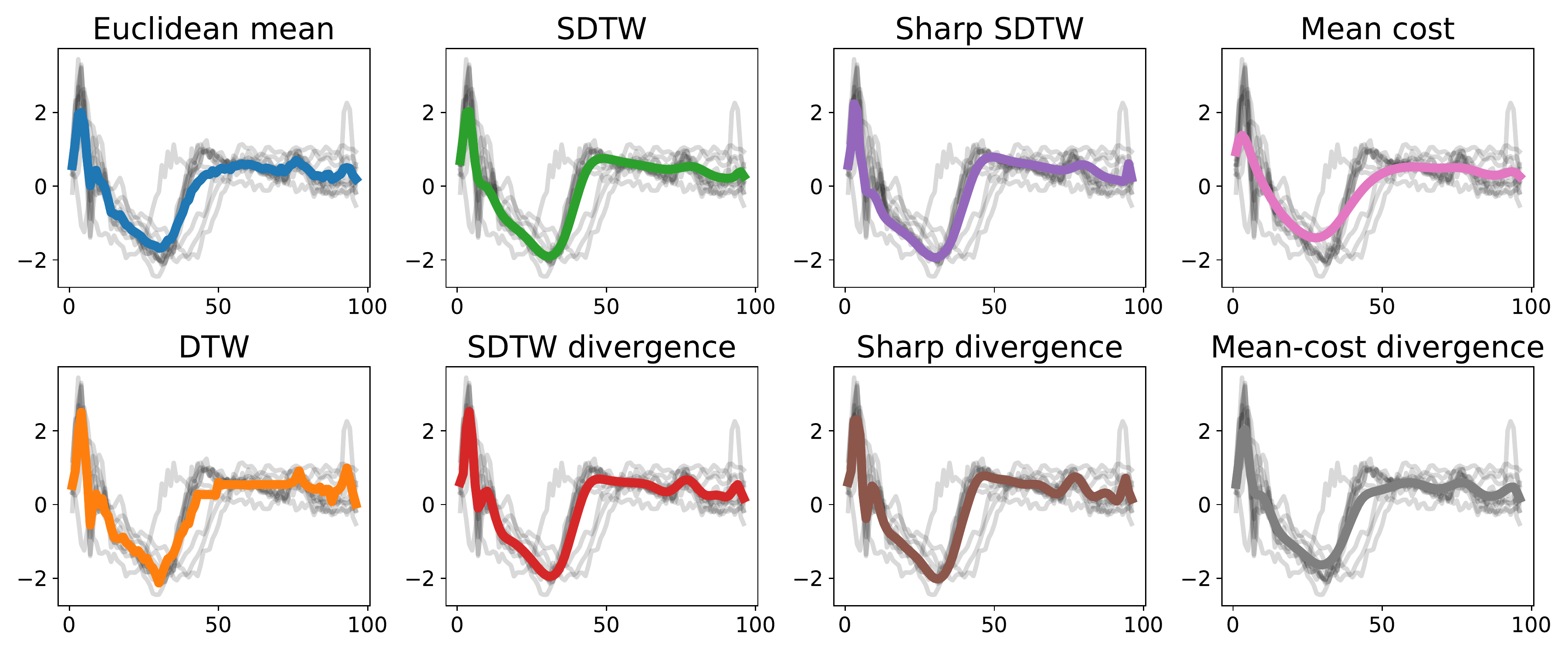}
\caption{Barycenters on the {\bf ECG200} dataset.}
\end{figure}

\begin{figure}[p]
\centering
\includegraphics[width=0.95 \textwidth]{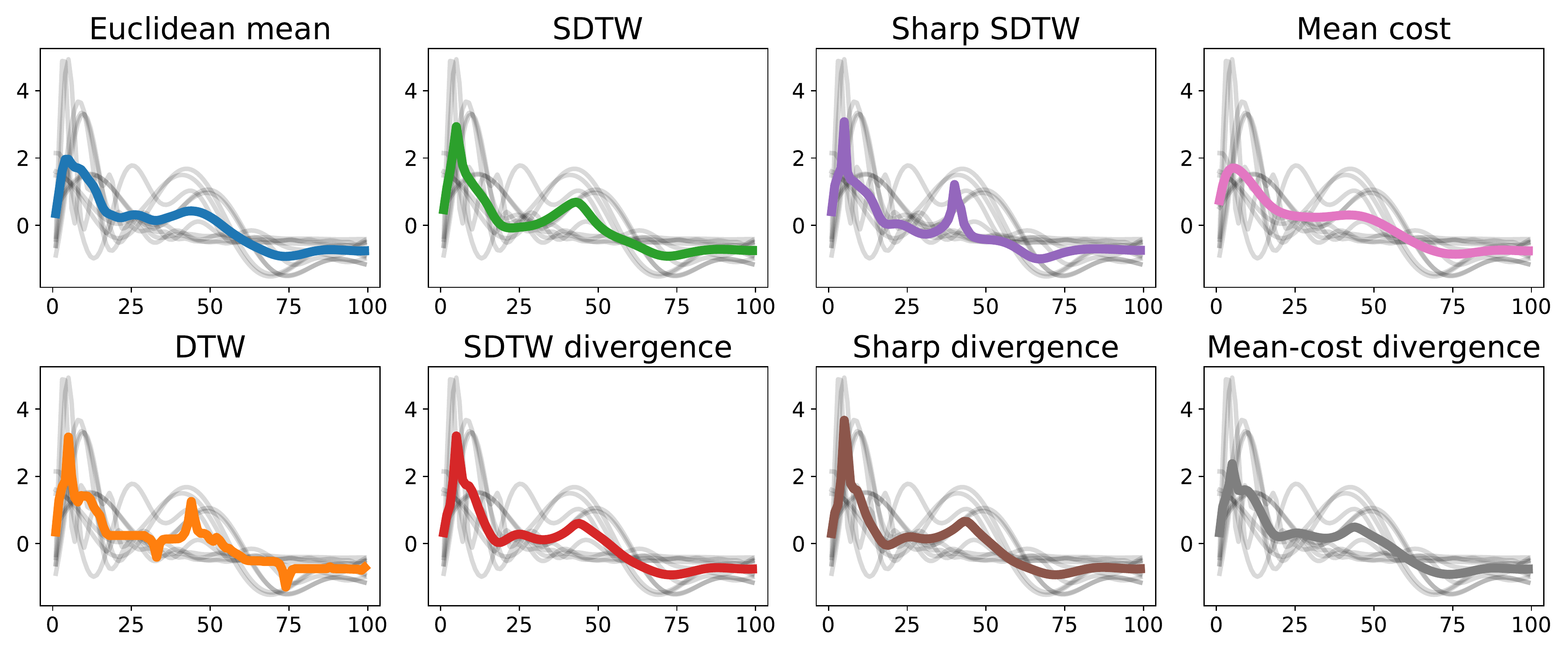}
\caption{Barycenters on the {\bf Medical Images} dataset.}
\end{figure}

\begin{figure}[p]
\centering
\includegraphics[width=0.95 \textwidth]{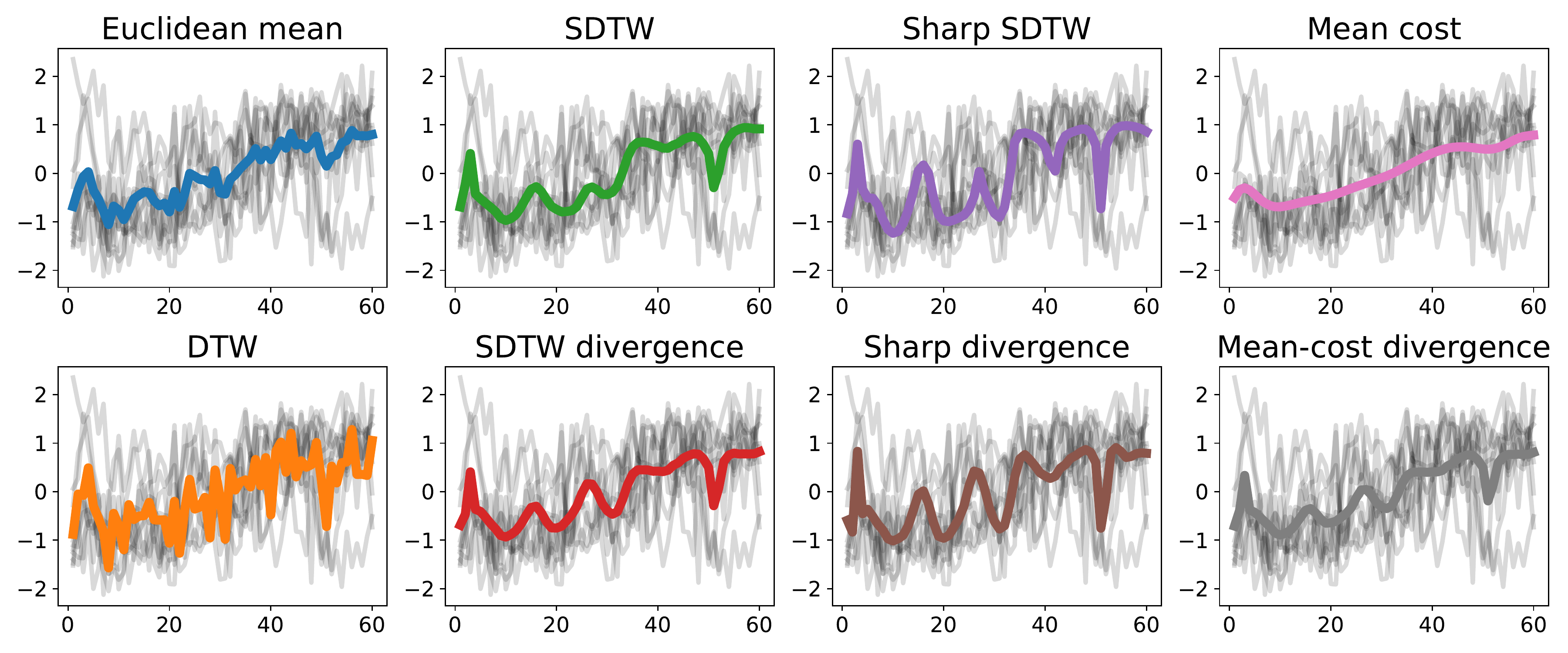}
\caption{Barycenters on the {\bf synthetic control} dataset.}
\end{figure}

\clearpage

\begin{table*}[t]
\caption{{\bf Three nearest neighbors results}. Each number indicates the
percentage of datasets in the UCR archive for which using $A$ in the nearest
neighbor classifier is within $99\%$ or better than using $B$ .}
\centering
\vspace{0.3em}
\begin{tabular}{lcccccccc}
\toprule
$A$ ($\downarrow$) vs. $B$ ($\rightarrow$)
 & Euc. & DTW & SDTW & SDTW div & Sharp & Sharp div & Mean cost & Mean-cost div \\[0.3em]
\midrule
Euc. &  -  & 39.29 & 29.49 & 31.17 & 37.18 & 28.00 & 95.24 & 65.48 \\[0.3em]
DTW & 70.24 &  -  & 53.85 & 45.45 & 57.69 & 42.67 & 90.48 & 83.33 \\[0.3em]
SDTW & 82.05 & 88.46 &  -  & 66.23 & 83.33 & 58.67 & 98.72 & 89.74 \\[0.3em]
SDTW div & 90.91 & 84.42 & 85.71 &  -  & 83.12 & 70.67 & 98.70 & 94.81 \\[0.3em]
Sharp & 78.21 & 82.05 & 64.10 & 58.44 &  -  & 53.33 & 98.72 & 87.18 \\[0.3em]
Sharp div & 86.67 & 90.67 & 81.33 & 77.33 & 89.33 &  -  & 98.67 & 96.00 \\[0.3em]
Mean cost & 8.33 & 13.10 & 6.41 & 3.90 & 5.13 & 4.00 &  -  & 44.05 \\[0.3em]
Mean-cost div & 46.43 & 34.52 & 24.36 & 20.78 & 24.36 & 21.33 & 98.81 &  -  \\[0.3em]
\hline
\end{tabular}
\end{table*}

\begin{table*}[t]
\caption{{\bf Five nearest neighbor results}. Each number indicates the
percentage of datasets in the UCR archive for which using $A$ in the nearest
neighbor classifier is within $99\%$ or better than using $B$ .}
\centering
\vspace{0.3em}
\begin{tabular}{lcccccccc}
\toprule
$A$ ($\downarrow$) vs. $B$ ($\rightarrow$)
 & Euc. & DTW & SDTW & SDTW div & Sharp & Sharp div & Mean cost & Mean-cost div \\[0.3em]
\midrule
Euc. &  -  & 40.48 & 30.77 & 28.57 & 33.33 & 24.68 & 95.29 & 70.24 \\[0.3em]
DTW & 73.81 &  -  & 48.72 & 44.16 & 55.13 & 45.45 & 88.10 & 83.33 \\[0.3em]
SDTW & 85.90 & 84.62 &  -  & 61.04 & 74.36 & 63.64 & 94.87 & 82.05 \\[0.3em]
SDTW div & 84.42 & 88.31 & 81.82 &  -  & 81.82 & 74.03 & 96.10 & 85.71 \\[0.3em]
Sharp & 85.90 & 87.18 & 70.51 & 58.44 &  -  & 59.74 & 97.44 & 82.05 \\[0.3em]
Sharp div & 90.91 & 84.42 & 80.52 & 76.62 & 84.42 &  -  & 96.10 & 87.01 \\[0.3em]
Mean cost & 10.59 & 13.10 & 10.26 & 7.79 & 7.69 & 7.79 &  -  & 45.24 \\[0.3em]
Mean-cost div & 45.24 & 32.14 & 26.92 & 20.78 & 26.92 & 19.48 & 98.81 &  -  \\[0.3em]
\hline
\end{tabular}
\end{table*}

\begin{table}[t]
\caption{{\bf Nearest neighbor classification accuracy with $k=1$.}}
\centering
\vspace{0.3em}
\begin{tiny}
\begin{tabular}{lcccccccc}
\toprule
Dataset name & Euc. & DTW & SDTW & SDTW div & Sharp & Sharp div & Mean cost & Mean-cost div \\[0.3em]
\midrule
50words & 63.08 & 69.01 & 80.66 & {\bf 81.54} & 79.12 & 79.78 & 58.90 & 67.91\\
Adiac & 61.13 & 60.36 & 61.38 & 71.36 & 60.10 & {\bf 72.12} & 28.39 & 54.48\\
ArrowHead & 80.00 & 70.29 & 77.14 & {\bf 81.71} & 80.57 & 79.43 & 72.57 & 78.86\\
Beef & {\bf 66.67} & 63.33 & 63.33 & 63.33 & 63.33 & 63.33 & 20.00 & 20.00\\
BeetleFly & {\bf 75.00} & 70.00 & 70.00 & 70.00 & 70.00 & {\bf 75.00} & 50.00 & 50.00\\
BirdChicken & 55.00 & {\bf 75.00} & {\bf 75.00} & {\bf 75.00} & {\bf 75.00} & {\bf 75.00} & 50.00 & 50.00\\
CBF & 85.22 & {\bf 99.67} & {\bf 99.67} & {\bf 99.67} & {\bf 99.67} & {\bf 99.67} & 78.78 & 95.00\\
Car & 73.33 & 73.33 & 73.33 & 75.00 & 75.00 & {\bf 78.33} & 23.33 & 23.33\\
ChlorineConcentration & 65.00 & 64.84 & 62.29 & 64.84 & 65.05 & {\bf 65.65} & 38.20 & 55.44\\
CinC\_ECG\_torso & 89.71 & 65.07 & 93.41 & 93.55 & 92.54 & {\bf 93.84} & 25.36 & 25.36\\
Coffee & {\bf 100.00} & {\bf 100.00} & {\bf 100.00} & {\bf 100.00} & {\bf 100.00} & {\bf 100.00} & 53.57 & 96.43\\
Computers & 57.60 & {\bf 70.00} & 69.60 & {\bf 70.00} & 69.20 & 67.20 & 50.00 & 50.00\\
Cricket\_X & 57.69 & 75.38 & 77.69 & {\bf 80.00} & 77.95 & 79.23 & 42.56 & 61.54\\
Cricket\_Y & 56.67 & 74.36 & 76.67 & {\bf 78.72} & 74.36 & 77.18 & 47.95 & 61.28\\
Cricket\_Z & 58.72 & 75.38 & 77.69 & {\bf 80.26} & 77.69 & 79.74 & 43.08 & 63.33\\
DiatomSizeReduction & 93.46 & {\bf 96.73} & 92.16 & 94.44 & 92.81 & 93.46 & 92.16 & 93.46\\
DistalPhalanxOutlineAgeGroup & 78.25 & 79.25 & 79.25 & 79.75 & 79.50 & {\bf 80.50} & 59.50 & 76.75\\
DistalPhalanxOutlineCorrect & 75.17 & 76.83 & {\bf 79.00} & 76.83 & 76.83 & 75.17 & 36.83 & 71.33\\
DistalPhalanxTW & 72.75 & 70.75 & 73.25 & 72.25 & {\bf 74.50} & 72.50 & 51.00 & 71.00\\
ECG200 & {\bf 88.00} & 77.00 & 86.00 & {\bf 88.00} & 82.00 & 87.00 & 87.00 & {\bf 88.00}\\
ECG5000 & 92.49 & 92.44 & {\bf 93.07} & 92.36 & 92.78 & 92.47 & 91.80 & 92.38\\
ECGFiveDays & 79.67 & 76.77 & 61.67 & {\bf 93.50} & 62.49 & 91.17 & 61.44 & 83.86\\
Earthquakes & 67.39 & 74.22 & {\bf 82.61} & 74.53 & {\bf 82.61} & 74.22 & 81.99 & 81.99\\
ElectricDevices & 54.93 & {\bf 60.02} & NA & NA & NA & NA & 26.17 & 59.12\\
FISH & 78.29 & 82.29 & 92.00 & {\bf 92.57} & 90.29 & 91.43 & 12.57 & 12.57\\
FaceAll & 71.36 & 80.77 & 74.38 & 82.31 & 76.27 & {\bf 82.78} & 25.33 & 81.89\\
FaceFour & 78.41 & 82.95 & 82.95 & {\bf 89.77} & 87.50 & {\bf 89.77} & 62.50 & 84.09\\
FacesUCR & 76.93 & 90.49 & 92.34 & {\bf 94.78} & 92.34 & 94.54 & 45.90 & 80.44\\
FordA & {\bf 65.90} & 56.21 & NA & NA & NA & NA & 51.26 & 51.26\\
FordB & 55.78 & {\bf 59.41} & 58.55 & NA & 58.83 & NA & 48.84 & 48.84\\
Gun\_Point & 91.33 & 90.67 & 97.33 & {\bf 98.00} & {\bf 98.00} & {\bf 98.00} & 82.00 & 90.00\\
Ham & 60.00 & 46.67 & 49.52 & 58.10 & 58.10 & {\bf 61.90} & 48.57 & 48.57\\
HandOutlines & {\bf 80.10} & 79.80 & NA & NA & NA & NA & 63.80 & 63.80\\
Haptics & 37.01 & 37.66 & 39.94 & 39.94 & 40.26 & {\bf 41.56} & 21.75 & 21.75\\
Herring & 51.56 & 53.12 & 57.81 & 57.81 & 60.94 & {\bf 62.50} & 59.38 & 59.38\\
InlineSkate & 34.18 & 38.36 & 42.55 & {\bf 43.09} & 42.00 & 42.36 & 15.64 & 15.64\\
InsectWingbeatSound & 56.16 & 35.51 & 55.05 & 56.87 & 56.26 & {\bf 57.07} & 54.55 & 56.97\\
ItalyPowerDemand & {\bf 95.53} & 95.04 & 93.68 & 95.04 & 94.07 & 95.43 & 90.38 & 94.95\\
LargeKitchenAppliances & 49.33 & 79.47 & {\bf 79.73} & {\bf 79.73} & {\bf 79.73} & {\bf 79.73} & 33.33 & 33.33\\
Lighting2 & 75.41 & 86.89 & {\bf 90.16} & 88.52 & {\bf 90.16} & 86.89 & 54.10 & 54.10\\
Lighting7 & 57.53 & 72.60 & 73.97 & 78.08 & 75.34 & {\bf 82.19} & 57.53 & 68.49\\
MALLAT & 91.43 & {\bf 93.39} & 89.72 & 91.39 & 90.62 & 92.24 & 12.54 & 12.54\\
Meat & 93.33 & 93.33 & {\bf 95.00} & 93.33 & {\bf 95.00} & 93.33 & 33.33 & 33.33\\
MedicalImages & 68.42 & 73.68 & 74.61 & 75.92 & 76.18 & {\bf 77.76} & 57.89 & 69.61\\
MiddlePhalanxOutlineAgeGroup & 74.00 & 75.00 & 71.00 & 73.25 & {\bf 75.25} & 73.75 & 66.25 & 73.25\\
MiddlePhalanxOutlineCorrect & 75.33 & 64.83 & 72.67 & {\bf 76.33} & 66.83 & 71.83 & 35.33 & 70.67\\
MiddlePhalanxTW & 56.14 & 58.40 & 58.40 & 58.40 & 58.40 & 58.40 & 52.63 & {\bf 59.15}\\
MoteStrain & 87.86 & 83.47 & 90.18 & 89.86 & {\bf 91.53} & 87.62 & 88.18 & 80.35\\
NonInvasiveFatalECG\_Thorax1 & {\bf 82.90} & 78.98 & NA & NA & NA & NA & 2.44 & 2.44\\
NonInvasiveFatalECG\_Thorax2 & {\bf 87.99} & 86.46 & NA & NA & NA & NA & 2.44 & 2.44\\
OSULeaf & 52.07 & 59.09 & {\bf 70.25} & 69.83 & {\bf 70.25} & 69.83 & 9.50 & 9.50\\
OliveOil & {\bf 86.67} & 83.33 & {\bf 86.67} & {\bf 86.67} & {\bf 86.67} & {\bf 86.67} & 16.67 & 16.67\\
PhalangesOutlinesCorrect & 76.11 & 72.61 & 74.59 & 77.04 & 71.91 & {\bf 77.39} & 42.31 & 73.08\\
Phoneme & 10.92 & 22.84 & {\bf 24.00} & 22.73 & 21.89 & 23.26 & 2.00 & 2.00\\
Plane & 96.19 & {\bf 100.00} & {\bf 100.00} & {\bf 100.00} & {\bf 100.00} & {\bf 100.00} & 84.76 & 96.19\\
ProximalPhalanxOutlineAgeGroup & 78.54 & 80.49 & 75.12 & {\bf 80.98} & {\bf 80.98} & {\bf 80.98} & 46.34 & 76.59\\
ProximalPhalanxOutlineCorrect & 80.76 & 77.66 & 79.04 & {\bf 83.51} & 74.23 & {\bf 83.51} & 31.96 & 73.20\\
ProximalPhalanxTW & 70.75 & 74.00 & 74.75 & 70.25 & {\bf 75.00} & 73.25 & 45.25 & 70.25\\
RefrigerationDevices & 39.47 & {\bf 46.40} & 45.87 & 44.80 & 45.60 & NA & 33.33 & 33.33\\
ScreenType & 36.00 & 40.00 & {\bf 41.33} & 40.27 & 39.47 & 39.47 & 33.33 & 33.33\\
ShapeletSim & 53.89 & 65.00 & 58.33 & {\bf 87.22} & 64.44 & 82.78 & 50.00 & 50.00\\
ShapesAll & 75.17 & 76.83 & 83.67 & {\bf 84.33} & 80.83 & 82.17 & 1.67 & 1.67\\
SmallKitchenAppliances & 34.40 & 64.27 & 66.67 & 66.67 & {\bf 67.47} & 65.87 & 33.33 & 33.33\\
SonyAIBORobotSurface & 69.55 & 72.55 & 72.55 & {\bf 76.71} & 72.55 & 76.54 & 45.42 & 76.04\\
SonyAIBORobotSurfaceII & {\bf 85.94} & 83.11 & 84.26 & 84.89 & 83.11 & 83.95 & 76.39 & 84.05\\
StarLightCurves & {\bf 84.88} & NA & NA & NA & NA & NA & 57.72 & NA\\
Strawberry & 93.80 & {\bf 93.96} & {\bf 93.96} & 93.80 & 93.80 & 93.64 & 79.45 & 93.80\\
SwedishLeaf & 78.88 & 79.20 & 82.40 & 88.16 & 82.24 & {\bf 89.12} & 46.72 & 79.84\\
Symbols & 89.95 & 94.97 & {\bf 96.18} & 95.38 & 95.18 & 95.28 & 86.93 & 90.15\\
ToeSegmentation1 & 67.98 & 77.19 & {\bf 83.33} & 82.89 & 80.26 & 81.58 & 63.16 & 63.16\\
ToeSegmentation2 & 80.77 & 83.85 & 90.77 & 86.15 & {\bf 92.31} & {\bf 92.31} & 79.23 & 83.85\\
Trace & 76.00 & {\bf 100.00} & {\bf 100.00} & {\bf 100.00} & {\bf 100.00} & {\bf 100.00} & 47.00 & 72.00\\
TwoLeadECG & 74.71 & {\bf 90.52} & {\bf 90.52} & 90.43 & 89.73 & 88.59 & 57.77 & 70.15\\
Two\_Patterns & 90.68 & {\bf 100.00} & {\bf 100.00} & {\bf 100.00} & {\bf 100.00} & {\bf 100.00} & 94.78 & 96.72\\
UWaveGestureLibraryAll & {\bf 94.81} & 89.17 & NA & NA & NA & NA & 12.53 & 12.53\\
Wine & 61.11 & 57.41 & 55.56 & {\bf 62.96} & 55.56 & {\bf 62.96} & 50.00 & 61.11\\
WordsSynonyms & 61.76 & 64.89 & 76.80 & {\bf 78.06} & 74.92 & 76.49 & 55.33 & 65.20\\
Worms & 36.46 & 46.41 & 47.51 & 48.07 & {\bf 49.17} & 42.54 & 41.99 & 41.99\\
WormsTwoClass & 58.56 & 66.30 & 55.80 & {\bf 67.40} & 57.46 & 64.09 & 41.99 & 41.99\\
synthetic\_control & 88.00 & {\bf 99.33} & 97.67 & {\bf 99.33} & {\bf 99.33} & {\bf 99.33} & 76.67 & 98.67\\
uWaveGestureLibrary\_X & 73.93 & 72.75 & 78.48 & {\bf 78.73} & 77.58 & 78.00 & 72.84 & 74.37\\
uWaveGestureLibrary\_Y & 66.16 & 63.40 & 70.30 & NA & 69.82 & {\bf 71.13} & 64.43 & 67.42\\
uWaveGestureLibrary\_Z & 64.96 & 65.83 & 68.51 & {\bf 69.65} & 68.06 & 68.90 & 62.90 & 64.91\\
wafer & 99.55 & 97.99 & 99.30 & 99.56 & 99.43 & {\bf 99.59} & 99.25 & 99.51\\
yoga & 83.03 & 83.67 & 83.97 & {\bf 85.30} & 84.70 & 83.57 & 46.43 & 46.43\\
\hline
\end{tabular}
\end{tiny}
\end{table}

\begin{table}[t]
\caption{{\bf Nearest neighbor classification accuracy with $k=3$.}}
\centering
\vspace{0.3em}
\begin{tiny}
\begin{tabular}{lcccccccc}
\toprule
Dataset name & Euc. & DTW & SDTW & SDTW div & Sharp & Sharp div & Mean cost & Mean-cost div \\[0.3em]
\midrule
50words & 61.98 & 66.37 & 80.22 & {\bf 80.66} & 77.80 & 78.90 & 59.34 & 66.81\\
Adiac & 55.24 & 57.29 & 56.78 & {\bf 69.05} & 54.99 & 66.50 & 26.34 & 49.10\\
ArrowHead & 79.43 & 70.86 & 80.57 & 79.43 & 78.86 & 82.86 & 72.57 & {\bf 84.57}\\
Beef & {\bf 60.00} & 56.67 & 53.33 & 56.67 & 56.67 & 56.67 & 20.00 & 20.00\\
BeetleFly & 65.00 & 70.00 & 50.00 & 65.00 & {\bf 75.00} & {\bf 75.00} & 50.00 & 50.00\\
BirdChicken & 45.00 & {\bf 60.00} & {\bf 60.00} & {\bf 60.00} & {\bf 60.00} & {\bf 60.00} & 50.00 & 50.00\\
CBF & 83.78 & {\bf 99.67} & {\bf 99.67} & {\bf 99.67} & {\bf 99.67} & {\bf 99.67} & 82.56 & 89.78\\
Car & {\bf 66.67} & 55.00 & 61.67 & {\bf 66.67} & 56.67 & 56.67 & 23.33 & 23.33\\
ChlorineConcentration & 56.59 & {\bf 56.69} & 56.12 & 56.54 & {\bf 56.69} & {\bf 56.69} & 38.44 & 51.54\\
CinC\_ECG\_torso & 85.22 & 49.78 & {\bf 86.67} & {\bf 86.67} & 85.87 & 85.58 & 24.78 & 24.78\\
Coffee & {\bf 100.00} & 92.86 & 92.86 & 92.86 & 92.86 & 92.86 & 53.57 & 92.86\\
Computers & 62.00 & {\bf 71.20} & {\bf 71.20} & {\bf 71.20} & {\bf 71.20} & {\bf 71.20} & 50.00 & 50.00\\
Cricket\_X & 51.79 & 74.36 & 75.38 & {\bf 77.44} & 72.56 & 75.13 & 42.05 & 55.38\\
Cricket\_Y & 50.51 & 70.51 & 71.03 & {\bf 76.41} & 71.03 & 73.33 & 44.62 & 56.92\\
Cricket\_Z & 54.62 & 75.38 & 77.95 & 78.72 & 76.92 & {\bf 78.97} & 42.31 & 59.23\\
DiatomSizeReduction & 89.22 & {\bf 92.81} & 89.22 & 89.87 & 89.87 & 89.87 & 87.58 & 89.54\\
DistalPhalanxOutlineAgeGroup & 78.50 & 83.50 & {\bf 83.75} & 79.75 & 83.25 & 79.25 & 59.25 & 79.25\\
DistalPhalanxOutlineCorrect & 75.83 & 79.83 & 79.33 & 79.83 & 79.83 & {\bf 80.67} & 36.67 & 74.33\\
DistalPhalanxTW & 75.75 & 73.00 & 72.75 & 75.00 & 75.00 & {\bf 76.75} & 53.75 & 72.75\\
ECG200 & {\bf 90.00} & 80.00 & 88.00 & 89.00 & 88.00 & 89.00 & 86.00 & 88.00\\
ECG5000 & 93.49 & 93.98 & 94.00 & 94.16 & 93.98 & {\bf 94.20} & 93.44 & 93.47\\
ECGFiveDays & 73.98 & 62.02 & 67.25 & 82.00 & 66.32 & {\bf 82.81} & 52.50 & 80.02\\
Earthquakes & 74.22 & 78.88 & 78.88 & 78.88 & 78.88 & 78.88 & {\bf 81.99} & {\bf 81.99}\\
ElectricDevices & 56.40 & {\bf 61.08} & NA & NA & NA & NA & 25.77 & 60.42\\
FISH & 75.43 & 79.43 & 90.29 & 90.29 & 90.29 & {\bf 91.43} & 12.57 & 12.57\\
FaceAll & 67.22 & 80.77 & 79.94 & 83.37 & 75.09 & {\bf 84.97} & 28.46 & 80.53\\
FaceFour & 65.91 & 68.18 & 68.18 & 72.73 & 59.09 & {\bf 77.27} & 46.59 & 69.32\\
FacesUCR & 67.76 & 88.63 & 90.44 & {\bf 93.90} & 91.32 & 93.41 & 47.17 & 71.32\\
FordA & {\bf 67.15} & 57.46 & NA & NA & NA & NA & 51.26 & 51.26\\
FordB & 58.33 & 61.83 & {\bf 61.94} & NA & 61.83 & NA & 51.16 & 51.16\\
Gun\_Point & 87.33 & 88.67 & 97.33 & {\bf 98.00} & {\bf 98.00} & {\bf 98.00} & 84.67 & 84.67\\
Ham & 59.05 & 51.43 & 52.38 & {\bf 62.86} & 57.14 & 61.90 & 51.43 & 51.43\\
HandOutlines & {\bf 84.90} & 81.00 & NA & NA & NA & NA & 63.80 & 63.80\\
Haptics & 38.64 & 42.86 & 41.23 & 41.56 & 37.01 & {\bf 43.51} & 21.75 & 21.75\\
Herring & 56.25 & 48.44 & 64.06 & 60.94 & 62.50 & {\bf 65.62} & 59.38 & 59.38\\
InlineSkate & 23.82 & 35.64 & 37.45 & {\bf 37.64} & 35.82 & 35.45 & 15.64 & 15.64\\
InsectWingbeatSound & {\bf 59.24} & 36.21 & 56.67 & 58.18 & 57.22 & 58.33 & 57.07 & 58.28\\
ItalyPowerDemand & {\bf 95.63} & 94.56 & 94.95 & 95.14 & 94.56 & 95.04 & 89.60 & 94.95\\
LargeKitchenAppliances & 45.60 & {\bf 80.00} & {\bf 80.00} & 77.60 & {\bf 80.00} & 77.07 & 33.33 & 33.33\\
Lighting2 & 77.05 & 86.89 & {\bf 91.80} & 90.16 & 83.61 & 85.25 & 45.90 & 45.90\\
Lighting7 & 60.27 & 71.23 & 79.45 & {\bf 82.19} & 78.08 & {\bf 82.19} & 57.53 & 71.23\\
MALLAT & 91.98 & 92.84 & 92.54 & {\bf 92.88} & 92.15 & 92.75 & 12.45 & 12.45\\
Meat & {\bf 93.33} & {\bf 93.33} & {\bf 93.33} & {\bf 93.33} & {\bf 93.33} & 91.67 & 33.33 & 33.33\\
MedicalImages & 67.76 & 70.92 & 72.11 & 73.42 & 72.76 & {\bf 74.61} & 57.24 & 69.21\\
MiddlePhalanxOutlineAgeGroup & 73.50 & {\bf 76.00} & {\bf 76.00} & 74.50 & {\bf 76.00} & {\bf 76.00} & 67.75 & 74.50\\
MiddlePhalanxOutlineCorrect & 77.17 & 72.17 & 74.50 & {\bf 77.67} & 73.67 & 76.00 & 35.50 & 75.33\\
MiddlePhalanxTW & 58.40 & 61.15 & 60.65 & 61.15 & 61.65 & {\bf 62.16} & 51.88 & 58.65\\
MoteStrain & 86.18 & 81.39 & 88.18 & 87.46 & {\bf 89.54} & 87.86 & 85.14 & 83.87\\
NonInvasiveFatalECG\_Thorax1 & {\bf 82.54} & 78.63 & NA & NA & NA & NA & 2.54 & 2.54\\
NonInvasiveFatalECG\_Thorax2 & {\bf 88.40} & 86.31 & NA & NA & NA & NA & 2.54 & 2.54\\
OSULeaf & 50.41 & 57.44 & 59.50 & 61.98 & 64.88 & {\bf 65.29} & 19.01 & 19.01\\
OliveOil & {\bf 90.00} & 86.67 & 86.67 & 86.67 & 86.67 & 86.67 & 40.00 & 40.00\\
PhalangesOutlinesCorrect & 77.97 & 75.41 & 76.57 & {\bf 79.37} & 76.57 & 79.14 & 42.07 & 73.66\\
Phoneme & 10.34 & 23.95 & 21.99 & 23.58 & 23.10 & {\bf 25.05} & 7.07 & 7.07\\
Plane & 96.19 & {\bf 100.00} & {\bf 100.00} & {\bf 100.00} & {\bf 100.00} & {\bf 100.00} & 84.76 & 96.19\\
ProximalPhalanxOutlineAgeGroup & {\bf 81.95} & 80.98 & 81.46 & 80.98 & {\bf 81.95} & {\bf 81.95} & 48.78 & 80.49\\
ProximalPhalanxOutlineCorrect & 84.88 & 83.16 & 81.79 & {\bf 85.57} & 78.01 & 84.19 & 31.62 & 74.91\\
ProximalPhalanxTW & 77.00 & {\bf 79.00} & 78.50 & 77.50 & 77.25 & 78.75 & 45.50 & 78.00\\
RefrigerationDevices & 39.20 & 46.40 & 46.13 & 45.87 & {\bf 46.67} & 46.13 & 33.33 & 33.33\\
ScreenType & 38.40 & 39.20 & {\bf 42.13} & 36.53 & 39.20 & 37.07 & 33.33 & 33.33\\
ShapeletSim & 52.78 & 62.78 & 62.78 & 80.00 & 68.33 & {\bf 81.67} & 50.00 & 50.00\\
ShapesAll & 69.00 & 71.00 & 77.33 & {\bf 77.67} & 75.67 & NA & 1.67 & 1.67\\
SmallKitchenAppliances & 36.53 & 67.47 & {\bf 70.67} & {\bf 70.67} & 67.73 & 67.20 & 33.33 & 33.33\\
SonyAIBORobotSurface & 57.40 & 61.73 & 61.73 & 61.73 & 61.73 & 61.73 & 43.59 & {\bf 67.22}\\
SonyAIBORobotSurfaceII & 79.85 & 80.27 & 77.65 & 79.12 & 79.01 & {\bf 80.90} & 76.50 & 80.06\\
StarLightCurves & {\bf 84.82} & NA & NA & NA & NA & NA & NA & NA\\
Strawberry & {\bf 92.33} & 91.84 & 91.68 & 92.01 & 90.05 & 91.03 & 78.96 & 90.38\\
SwedishLeaf & 71.84 & 77.92 & 80.48 & 86.56 & 78.88 & {\bf 87.36} & 47.84 & 77.44\\
Symbols & 85.03 & 92.86 & {\bf 96.18} & {\bf 96.18} & 95.98 & 96.08 & 81.91 & 86.13\\
ToeSegmentation1 & 60.53 & 75.44 & {\bf 82.02} & 77.63 & 75.88 & 78.51 & 57.46 & 63.60\\
ToeSegmentation2 & 82.31 & 81.54 & 89.23 & 89.23 & 91.54 & {\bf 93.08} & 82.31 & 86.15\\
Trace & 65.00 & {\bf 100.00} & {\bf 100.00} & {\bf 100.00} & {\bf 100.00} & {\bf 100.00} & 47.00 & 64.00\\
TwoLeadECG & 63.48 & 85.16 & {\bf 85.34} & 63.48 & 82.44 & 63.74 & 55.66 & 63.21\\
Two\_Patterns & 85.95 & {\bf 100.00} & {\bf 100.00} & {\bf 100.00} & {\bf 100.00} & {\bf 100.00} & 90.72 & 94.20\\
UWaveGestureLibraryAll & {\bf 94.39} & 89.53 & NA & NA & NA & NA & 12.62 & 12.62\\
Wine & 55.56 & 57.41 & {\bf 62.96} & {\bf 62.96} & 51.85 & 61.11 & 50.00 & 61.11\\
WordsSynonyms & 56.74 & 59.56 & {\bf 72.41} & 69.59 & 70.85 & 72.10 & 54.23 & 59.56\\
Worms & 36.46 & {\bf 42.54} & {\bf 42.54} & {\bf 42.54} & {\bf 42.54} & {\bf 42.54} & 13.81 & 13.81\\
WormsTwoClass & 59.12 & 64.09 & {\bf 70.17} & {\bf 70.17} & 65.19 & 65.19 & 58.01 & 58.01\\
synthetic\_control & 91.00 & 98.33 & 98.33 & 98.33 & 98.33 & 98.33 & 74.67 & {\bf 98.67}\\
uWaveGestureLibrary\_X & 73.03 & 73.73 & 78.00 & {\bf 78.31} & 76.97 & 77.41 & 71.94 & 73.84\\
uWaveGestureLibrary\_Y & 66.67 & 63.18 & 70.63 & {\bf 71.36} & 70.18 & NA & 65.47 & 67.17\\
uWaveGestureLibrary\_Z & 65.75 & 66.78 & 68.37 & {\bf 69.43} & 67.87 & 68.87 & 64.38 & 66.50\\
wafer & 99.38 & 97.52 & 99.06 & 99.42 & 99.06 & {\bf 99.45} & 99.06 & {\bf 99.45}\\
yoga & 79.23 & 82.17 & {\bf 82.53} & 82.33 & 82.23 & 82.33 & 46.43 & 46.43\\
\hline
\end{tabular}
\end{tiny}
\end{table}

\begin{table}[t]
\caption{{\bf Nearest neighbor classification accuracy with $k=5$.}}
\centering
\vspace{0.3em}
\begin{tiny}
\begin{tabular}{lcccccccc}
\toprule
Dataset name & Euc. & DTW & SDTW & SDTW div & Sharp & Sharp div & Mean cost & Mean-cost div \\[0.3em]
\midrule
50words & 61.98 & 66.15 & 77.80 & {\bf 79.12} & 75.60 & 77.80 & 57.80 & 65.93\\
Adiac & 52.17 & 53.20 & 59.34 & {\bf 63.68} & 55.75 & 61.64 & 25.06 & 46.55\\
ArrowHead & 66.86 & {\bf 68.57} & 62.86 & 64.57 & 63.43 & 66.86 & 62.29 & {\bf 68.57}\\
Beef & {\bf 50.00} & 43.33 & 46.67 & 43.33 & 43.33 & 43.33 & 20.00 & 20.00\\
BeetleFly & 60.00 & 70.00 & 60.00 & 65.00 & 70.00 & {\bf 80.00} & 50.00 & 50.00\\
BirdChicken & 55.00 & 65.00 & 70.00 & {\bf 75.00} & 65.00 & 60.00 & 50.00 & 50.00\\
CBF & 76.67 & {\bf 98.22} & {\bf 98.22} & {\bf 98.22} & {\bf 98.22} & {\bf 98.22} & 75.56 & 88.78\\
Car & 63.33 & 50.00 & {\bf 66.67} & {\bf 66.67} & 63.33 & {\bf 66.67} & 31.67 & 31.67\\
ChlorineConcentration & {\bf 54.87} & 54.82 & {\bf 54.87} & {\bf 54.87} & 54.82 & 54.66 & 44.32 & 51.46\\
CinC\_ECG\_torso & 77.39 & 42.61 & 80.14 & 80.22 & 82.46 & {\bf 83.26} & 24.78 & 24.78\\
Coffee & {\bf 96.43} & {\bf 96.43} & {\bf 96.43} & {\bf 96.43} & {\bf 96.43} & {\bf 96.43} & 60.71 & {\bf 96.43}\\
Computers & 60.40 & 68.80 & {\bf 69.60} & 68.40 & {\bf 69.60} & 68.00 & 50.00 & 50.00\\
Cricket\_X & 48.21 & 72.56 & 71.79 & 71.79 & 71.54 & {\bf 72.82} & 40.77 & 57.18\\
Cricket\_Y & 50.26 & 68.46 & 68.46 & 71.79 & 68.72 & {\bf 73.85} & 42.82 & 55.64\\
Cricket\_Z & 49.49 & 76.67 & 77.18 & 79.49 & 76.15 & {\bf 80.26} & 39.23 & 58.46\\
DiatomSizeReduction & 86.93 & 70.92 & 85.62 & 85.62 & 80.07 & 78.43 & {\bf 87.25} & 86.93\\
DistalPhalanxOutlineAgeGroup & 79.75 & {\bf 83.50} & {\bf 83.50} & {\bf 83.50} & {\bf 83.50} & 82.75 & 60.50 & 80.00\\
DistalPhalanxOutlineCorrect & 76.33 & 78.17 & 79.17 & 78.17 & 78.17 & {\bf 79.67} & 35.83 & 74.83\\
DistalPhalanxTW & 76.75 & 76.25 & 78.25 & 78.00 & 76.50 & {\bf 79.00} & 53.25 & 73.50\\
ECG200 & {\bf 90.00} & 79.00 & 86.00 & 87.00 & 87.00 & 88.00 & 85.00 & 89.00\\
ECG5000 & 93.91 & 93.84 & {\bf 94.33} & 93.84 & 94.24 & 93.84 & 93.89 & 93.87\\
ECGFiveDays & 61.21 & 60.16 & 75.38 & 77.82 & 68.99 & {\bf 77.93} & 51.34 & 77.00\\
Earthquakes & 78.57 & 79.19 & 79.19 & 79.19 & 79.19 & 79.19 & {\bf 81.99} & {\bf 81.99}\\
ElectricDevices & 58.38 & {\bf 61.03} & NA & NA & NA & NA & 27.19 & 60.80\\
FISH & 72.00 & 73.14 & 89.14 & 90.86 & 90.86 & {\bf 91.43} & 16.57 & 16.57\\
FaceAll & 64.62 & 81.01 & 71.66 & {\bf 85.03} & 74.44 & 80.89 & 30.59 & 79.59\\
FaceFour & 52.27 & {\bf 68.18} & {\bf 68.18} & {\bf 68.18} & 44.32 & 67.05 & 42.05 & 50.00\\
FacesUCR & 62.20 & 86.20 & 88.20 & {\bf 92.78} & 89.61 & 91.76 & 45.07 & 67.22\\
FordA & {\bf 68.62} & 58.71 & NA & NA & NA & NA & 51.26 & 51.26\\
FordB & 58.33 & 63.97 & {\bf 64.11} & NA & 63.28 & NA & 48.84 & 48.84\\
Gun\_Point & 80.00 & 82.67 & 92.67 & {\bf 94.67} & 92.00 & 92.67 & 81.33 & 80.67\\
Ham & 62.86 & 53.33 & 60.95 & 63.81 & 62.86 & {\bf 64.76} & 51.43 & 51.43\\
HandOutlines & {\bf 85.10} & 81.40 & NA & NA & NA & NA & 63.80 & 63.80\\
Haptics & 41.56 & 41.23 & {\bf 51.30} & 50.97 & 47.73 & 49.03 & 19.16 & 19.16\\
Herring & 51.56 & 54.69 & 54.69 & 56.25 & {\bf 59.38} & 56.25 & {\bf 59.38} & {\bf 59.38}\\
InlineSkate & 22.55 & 33.27 & {\bf 37.64} & 33.82 & 33.45 & 33.45 & 15.45 & 15.45\\
InsectWingbeatSound & {\bf 59.90} & 35.45 & 57.27 & 59.55 & 56.67 & 59.80 & 56.01 & 59.65\\
ItalyPowerDemand & {\bf 95.24} & 94.36 & 95.04 & 94.46 & 95.04 & 94.46 & 88.34 & 94.46\\
LargeKitchenAppliances & 45.60 & 78.67 & {\bf 78.93} & 78.67 & 78.67 & 75.47 & 33.33 & 33.33\\
Lighting2 & 72.13 & 81.97 & {\bf 85.25} & 83.61 & {\bf 85.25} & {\bf 85.25} & 54.10 & 54.10\\
Lighting7 & 57.53 & 75.34 & 76.71 & 75.34 & {\bf 79.45} & 75.34 & 49.32 & 63.01\\
MALLAT & 78.89 & {\bf 82.77} & 81.32 & 81.75 & 80.68 & 81.49 & 12.54 & 12.54\\
Meat & 91.67 & {\bf 93.33} & 91.67 & 90.00 & 90.00 & {\bf 93.33} & 33.33 & 33.33\\
MedicalImages & 66.05 & 69.74 & {\bf 71.45} & {\bf 71.45} & 71.18 & 71.32 & 54.74 & 69.47\\
MiddlePhalanxOutlineAgeGroup & 76.50 & 76.75 & 76.75 & 75.50 & 76.25 & {\bf 77.25} & 68.00 & 74.50\\
MiddlePhalanxOutlineCorrect & 76.00 & 74.50 & 74.33 & 77.17 & 74.50 & {\bf 77.50} & 35.67 & 74.67\\
MiddlePhalanxTW & 62.16 & 62.91 & 60.15 & 61.15 & {\bf 63.66} & 60.65 & 51.38 & 59.90\\
MoteStrain & 85.14 & 82.43 & 87.54 & 85.62 & {\bf 88.82} & 88.18 & 83.95 & 82.91\\
NonInvasiveFatalECG\_Thorax1 & {\bf 82.60} & 78.78 & NA & NA & NA & NA & 2.90 & 2.90\\
NonInvasiveFatalECG\_Thorax2 & {\bf 88.65} & 85.24 & NA & NA & NA & NA & 2.90 & 2.90\\
OSULeaf & 47.11 & 54.55 & 57.44 & 58.26 & {\bf 64.46} & 62.40 & 18.18 & 18.18\\
OliveOil & {\bf 83.33} & 73.33 & 80.00 & 80.00 & 80.00 & 76.67 & 40.00 & 40.00\\
PhalangesOutlinesCorrect & 77.86 & 75.64 & 78.55 & {\bf 79.60} & 77.16 & 79.37 & 42.89 & 75.87\\
Phoneme & 12.03 & 24.95 & 25.95 & 25.58 & 24.74 & {\bf 26.85} & 7.07 & 7.07\\
Plane & 96.19 & {\bf 100.00} & {\bf 100.00} & {\bf 100.00} & {\bf 100.00} & {\bf 100.00} & 83.81 & 96.19\\
ProximalPhalanxOutlineAgeGroup & 82.44 & 82.44 & 83.41 & 83.41 & 82.93 & {\bf 85.85} & 48.78 & 81.46\\
ProximalPhalanxOutlineCorrect & 84.19 & 80.76 & 84.54 & {\bf 86.94} & 80.07 & 86.25 & 31.62 & 79.38\\
ProximalPhalanxTW & 79.75 & 79.50 & 79.00 & 78.75 & 79.25 & 79.25 & 45.00 & {\bf 80.25}\\
RefrigerationDevices & 38.93 & {\bf 48.27} & 46.40 & {\bf 48.27} & 47.47 & 47.47 & 33.33 & 33.33\\
ScreenType & 41.60 & {\bf 42.67} & 42.13 & 40.53 & {\bf 42.67} & 39.20 & 33.33 & 33.33\\
ShapeletSim & 54.44 & 63.89 & 63.89 & 72.22 & 63.89 & {\bf 76.67} & 50.00 & 50.00\\
ShapesAll & 65.83 & 68.17 & 72.00 & 72.83 & 72.83 & {\bf 73.33} & 1.67 & 1.67\\
SmallKitchenAppliances & 36.53 & 68.00 & 68.00 & 67.73 & {\bf 68.80} & 68.27 & 33.33 & 33.33\\
SonyAIBORobotSurface & 46.92 & 52.25 & 52.25 & 52.25 & 52.25 & 52.25 & 42.93 & {\bf 56.57}\\
SonyAIBORobotSurfaceII & 77.12 & 77.65 & 74.29 & 76.92 & 77.33 & 77.75 & 75.13 & {\bf 79.33}\\
StarLightCurves & {\bf 84.51} & NA & NA & NA & NA & NA & 57.72 & NA\\
Strawberry & {\bf 92.33} & 91.68 & 87.77 & 90.86 & 91.19 & 90.54 & 79.45 & 89.40\\
SwedishLeaf & 71.84 & 78.72 & 78.24 & 85.12 & 77.76 & {\bf 85.44} & 48.48 & 78.88\\
Symbols & 73.37 & 90.45 & 93.47 & 77.39 & {\bf 94.37} & 77.89 & 71.36 & 76.58\\
ToeSegmentation1 & 61.40 & 71.49 & 72.81 & {\bf 76.32} & 73.25 & 72.81 & 58.33 & 61.40\\
ToeSegmentation2 & 84.62 & 83.08 & 83.85 & 84.62 & 85.38 & 84.62 & 84.62 & {\bf 86.92}\\
Trace & 54.00 & {\bf 100.00} & {\bf 100.00} & {\bf 100.00} & {\bf 100.00} & {\bf 100.00} & 49.00 & 53.00\\
TwoLeadECG & 59.70 & 81.39 & 74.54 & {\bf 81.56} & 72.61 & 72.87 & 55.14 & 60.76\\
Two\_Patterns & 82.50 & {\bf 100.00} & {\bf 100.00} & {\bf 100.00} & {\bf 100.00} & {\bf 100.00} & 87.62 & 91.52\\
UWaveGestureLibraryAll & {\bf 93.89} & 89.06 & NA & NA & NA & NA & 12.67 & 12.67\\
Wine & 53.70 & 48.15 & 59.26 & 51.85 & {\bf 66.67} & 59.26 & 50.00 & 53.70\\
WordsSynonyms & 54.70 & 55.33 & 67.40 & 64.89 & 66.93 & {\bf 68.03} & 51.88 & 58.62\\
Worms & 38.12 & 44.20 & 49.17 & {\bf 50.28} & 46.96 & 48.62 & 13.81 & 13.81\\
WormsTwoClass & 60.22 & 66.85 & {\bf 70.72} & {\bf 70.72} & 67.40 & 67.96 & 58.01 & 58.01\\
synthetic\_control & 87.00 & 97.33 & 97.33 & 97.33 & 97.33 & 97.33 & 76.00 & {\bf 98.67}\\
uWaveGestureLibrary\_X & 72.89 & 73.73 & 77.22 & {\bf 77.69} & 76.52 & 77.05 & 71.50 & 73.73\\
uWaveGestureLibrary\_Y & 66.36 & 64.10 & 70.46 & {\bf 71.08} & 69.74 & 70.71 & 65.75 & 67.59\\
uWaveGestureLibrary\_Z & 65.97 & 67.11 & 68.79 & 68.90 & 68.57 & {\bf 69.29} & 64.82 & 66.22\\
wafer & {\bf 99.17} & 97.13 & 98.91 & 99.01 & 99.01 & 99.08 & 98.78 & 99.08\\
yoga & 75.63 & 78.53 & 78.40 & {\bf 78.70} & 78.27 & 78.57 & 46.43 & 46.43\\
\hline
\end{tabular}
\end{tiny}
\end{table}

\begin{table}[t]
\caption{{\bf Nearest centroid classification accuracy.}}
\centering
\vspace{0.3em}
\begin{tiny}
\begin{tabular}{lcccccccc}
\toprule
Dataset name & Euc. & DTW & SDTW & SDTW div & Sharp & Sharp div & Mean cost & Mean-cost div \\[0.3em]
\midrule
50words & 51.65 & 59.78 & 76.26 & {\bf 78.02} & 69.45 & 76.70 & 50.33 & 51.21\\
Adiac & 54.99 & 47.06 & 67.52 & {\bf 68.54} & 66.75 & 67.26 & 44.25 & 46.55\\
ArrowHead & {\bf 61.14} & 50.86 & 51.43 & 57.71 & 49.71 & {\bf 61.14} & 58.86 & 59.43\\
Beef & {\bf 53.33} & 43.33 & 46.67 & 36.67 & 43.33 & 46.67 & 20.00 & 20.00\\
BeetleFly & {\bf 85.00} & 80.00 & 70.00 & 70.00 & 80.00 & 70.00 & 50.00 & 50.00\\
BirdChicken & 55.00 & 60.00 & {\bf 65.00} & 60.00 & 60.00 & 60.00 & 50.00 & 50.00\\
CBF & 76.33 & 96.89 & {\bf 97.11} & {\bf 97.11} & 97.00 & 97.00 & 73.00 & 74.44\\
Car & 61.67 & 61.67 & 70.00 & 73.33 & 73.33 & {\bf 75.00} & 23.33 & 23.33\\
ChlorineConcentration & 33.31 & 32.45 & {\bf 35.23} & 32.19 & 31.98 & 33.41 & 34.82 & 34.95\\
CinC\_ECG\_torso & 38.55 & 40.29 & {\bf 71.88} & 70.36 & 59.49 & 64.42 & 25.36 & 25.36\\
Coffee & {\bf 96.43} & {\bf 96.43} & {\bf 96.43} & {\bf 96.43} & {\bf 96.43} & {\bf 96.43} & 89.29 & 89.29\\
Computers & 41.60 & {\bf 63.20} & 51.60 & 56.80 & 62.80 & {\bf 63.20} & 50.00 & 50.00\\
Cricket\_X & 23.85 & 57.69 & 56.92 & 56.67 & 58.46 & {\bf 58.97} & 25.64 & 26.15\\
Cricket\_Y & 34.87 & 52.56 & {\bf 55.64} & 54.87 & 53.59 & 55.13 & 33.59 & 33.59\\
Cricket\_Z & 30.51 & 60.00 & 61.03 & 60.00 & 58.21 & {\bf 62.31} & 30.26 & 30.26\\
DiatomSizeReduction & 95.75 & 95.10 & {\bf 96.73} & 96.41 & 96.08 & 95.42 & 94.44 & 95.42\\
DistalPhalanxOutlineAgeGroup & 81.75 & 84.00 & 84.50 & 84.75 & 84.50 & {\bf 85.00} & 80.25 & 81.25\\
DistalPhalanxOutlineCorrect & 47.17 & {\bf 48.17} & 48.00 & 47.33 & 47.00 & 47.17 & {\bf 48.17} & 47.17\\
DistalPhalanxTW & 74.75 & {\bf 75.75} & 74.50 & 74.50 & 74.50 & 73.00 & 73.00 & 72.75\\
ECG200 & {\bf 75.00} & {\bf 75.00} & 72.00 & 73.00 & 69.00 & 73.00 & 74.00 & 74.00\\
ECG5000 & 86.04 & 84.53 & {\bf 86.73} & 85.98 & 86.02 & 86.09 & 81.44 & 83.64\\
ECGFiveDays & 68.99 & 65.27 & 80.60 & 83.39 & 80.95 & {\bf 85.60} & 79.56 & 80.26\\
Earthquakes & 75.47 & 58.07 & {\bf 82.30} & 65.22 & 71.12 & 72.98 & 81.99 & 81.99\\
ElectricDevices & 48.27 & 53.60 & 57.07 & {\bf 61.57} & 53.61 & 51.28 & 50.55 & 50.37\\
FISH & 56.00 & 65.71 & 81.14 & {\bf 84.00} & 81.14 & 82.86 & 13.71 & 13.71\\
FaceAll & 49.17 & 80.71 & 81.60 & 88.58 & 85.98 & {\bf 89.17} & 58.88 & 64.56\\
FaceFour & 84.09 & 82.95 & 86.36 & 89.77 & 88.64 & {\bf 90.91} & 78.41 & 77.27\\
FacesUCR & 53.95 & 79.22 & 88.98 & 91.07 & 90.78 & {\bf 91.85} & 57.37 & 59.46\\
FordA & 49.60 & 55.57 & 55.62 & 52.43 & 54.96 & {\bf 56.32} & 51.26 & 51.26\\
FordB & 49.97 & {\bf 60.70} & 47.58 & 55.94 & 58.33 & 54.81 & 51.16 & 51.16\\
Gun\_Point & 75.33 & 68.00 & 82.00 & 81.33 & {\bf 92.00} & 86.00 & 68.67 & 71.33\\
Ham & 76.19 & 73.33 & 71.43 & 75.24 & {\bf 79.05} & 72.38 & 48.57 & 48.57\\
HandOutlines & 81.80 & 79.20 & {\bf 82.40} & NA & NA & NA & 36.20 & 36.20\\
Haptics & 39.29 & 35.71 & 46.10 & 46.10 & {\bf 48.38} & 47.73 & 19.48 & 19.48\\
Herring & 54.69 & 60.94 & {\bf 64.06} & {\bf 64.06} & 59.38 & 62.50 & 59.38 & 59.38\\
InlineSkate & 19.27 & 22.73 & 23.45 & {\bf 26.36} & 22.73 & 21.45 & 9.64 & 9.64\\
InsectWingbeatSound & {\bf 60.10} & 29.80 & 58.18 & 58.64 & 58.43 & 58.79 & 58.43 & 58.38\\
ItalyPowerDemand & {\bf 91.84} & 74.15 & 88.14 & 90.48 & 85.62 & 87.37 & 71.62 & 84.35\\
LargeKitchenAppliances & 44.00 & 71.47 & 72.00 & 73.60 & {\bf 74.67} & 72.53 & 33.33 & 33.33\\
Lighting2 & 68.85 & 62.30 & 67.21 & {\bf 72.13} & 65.57 & 62.30 & 45.90 & 45.90\\
Lighting7 & 58.90 & 72.60 & 78.08 & {\bf 83.56} & 56.16 & 58.90 & 61.64 & 63.01\\
MALLAT & {\bf 96.67} & 94.93 & 95.74 & 94.84 & 94.80 & 94.88 & 12.54 & 12.54\\
Meat & {\bf 93.33} & {\bf 93.33} & 85.00 & 85.00 & 90.00 & 85.00 & 33.33 & 33.33\\
MedicalImages & 38.55 & 44.21 & 40.39 & 40.92 & {\bf 45.53} & 45.00 & 32.11 & 33.55\\
MiddlePhalanxOutlineAgeGroup & 73.25 & 72.50 & 72.75 & 72.75 & 72.75 & {\bf 75.25} & 73.75 & 73.25\\
MiddlePhalanxOutlineCorrect & {\bf 55.17} & 48.50 & 52.17 & 52.83 & 51.83 & 52.83 & 51.83 & 52.83\\
MiddlePhalanxTW & 59.15 & 56.64 & 58.15 & 58.15 & 58.90 & 58.65 & {\bf 59.40} & {\bf 59.40}\\
MoteStrain & 86.10 & 82.43 & {\bf 90.42} & 90.18 & 82.27 & 88.82 & 82.99 & 83.87\\
NonInvasiveFatalECG\_Thorax1 & 76.95 & 70.13 & 81.63 & {\bf 82.29} & 81.12 & NA & 2.44 & 2.44\\
NonInvasiveFatalECG\_Thorax2 & 80.20 & 76.28 & 87.23 & 87.68 & {\bf 87.74} & NA & 2.44 & 2.44\\
OSULeaf & 35.95 & 45.87 & {\bf 52.07} & 51.24 & 50.00 & 50.41 & 13.22 & 13.22\\
OliveOil & {\bf 86.67} & 76.67 & 83.33 & {\bf 86.67} & 83.33 & 83.33 & 16.67 & 16.67\\
PhalangesOutlinesCorrect & 62.59 & 63.64 & 63.75 & {\bf 64.45} & {\bf 64.45} & 63.99 & 61.42 & 62.47\\
Phoneme & 7.86 & 17.67 & 20.15 & 20.57 & 19.83 & {\bf 20.99} & 2.00 & 2.00\\
Plane & 96.19 & 99.05 & 99.05 & 99.05 & {\bf 100.00} & {\bf 100.00} & 95.24 & 96.19\\
ProximalPhalanxOutlineAgeGroup & 81.95 & 82.93 & {\bf 84.39} & {\bf 84.39} & {\bf 84.39} & 83.90 & 81.46 & 80.49\\
ProximalPhalanxOutlineCorrect & 64.60 & {\bf 64.95} & {\bf 64.95} & {\bf 64.95} & {\bf 64.95} & {\bf 64.95} & 64.26 & 64.60\\
ProximalPhalanxTW & 70.75 & 73.50 & 81.25 & {\bf 81.50} & 80.00 & 80.75 & 69.75 & 68.50\\
RefrigerationDevices & 35.47 & 57.87 & 58.13 & 55.20 & {\bf 61.60} & 58.13 & 33.33 & 33.33\\
ScreenType & {\bf 44.27} & 38.13 & 37.33 & 40.00 & 37.60 & 40.80 & 33.33 & 33.33\\
ShapeletSim & 50.00 & 61.67 & {\bf 73.33} & 72.78 & 57.22 & 68.89 & 50.00 & 50.00\\
ShapesAll & 51.33 & 62.17 & 65.50 & {\bf 68.67} & 64.50 & 66.83 & 1.67 & 1.67\\
SmallKitchenAppliances & 41.87 & 64.53 & 68.00 & {\bf 68.80} & 65.87 & 64.53 & 33.33 & 33.33\\
SonyAIBORobotSurface & 81.20 & {\bf 82.86} & 82.70 & {\bf 82.86} & 80.37 & 81.53 & 80.70 & 78.70\\
SonyAIBORobotSurfaceII & 79.33 & 76.60 & 79.85 & 76.50 & {\bf 80.27} & 78.91 & 77.12 & 76.92\\
StarLightCurves & 76.17 & 82.93 & {\bf 83.57} & 83.35 & 81.64 & NA & 14.29 & 14.29\\
Strawberry & 66.88 & 61.17 & 65.58 & 68.84 & 67.54 & {\bf 72.43} & 65.74 & 65.58\\
SwedishLeaf & 70.24 & 70.40 & 79.36 & {\bf 81.12} & 77.12 & 80.00 & 71.36 & 71.52\\
Symbols & 86.43 & 95.78 & 95.08 & 95.58 & 95.58 & {\bf 96.08} & 88.74 & 87.84\\
ToeSegmentation1 & 57.46 & 62.72 & 73.25 & 71.05 & 69.30 & {\bf 74.56} & 52.63 & 54.39\\
ToeSegmentation2 & 54.62 & {\bf 86.92} & 86.15 & 85.38 & 80.77 & 84.62 & 55.38 & 54.62\\
Trace & 58.00 & 98.00 & 98.00 & 97.00 & {\bf 99.00} & {\bf 99.00} & 56.00 & 57.00\\
TwoLeadECG & 55.49 & 76.21 & 78.05 & 83.06 & 78.49 & {\bf 89.38} & 57.33 & 57.16\\
Two\_Patterns & 46.48 & 98.40 & {\bf 98.65} & 98.18 & 98.42 & 98.55 & 56.30 & 50.75\\
UWaveGestureLibraryAll & 84.95 & 83.45 & 89.31 & {\bf 90.90} & 90.09 & NA & 12.20 & 12.20\\
Wine & 55.56 & 53.70 & {\bf 57.41} & 55.56 & {\bf 57.41} & 55.56 & 55.56 & 55.56\\
WordsSynonyms & 27.12 & 34.33 & {\bf 52.19} & 51.72 & 49.84 & 50.78 & 26.33 & 26.49\\
Worms & 21.55 & 40.33 & 43.65 & {\bf 44.75} & 42.54 & 42.54 & 41.99 & 41.99\\
WormsTwoClass & 54.14 & 62.98 & 67.96 & {\bf 70.72} & 65.19 & 56.91 & 41.99 & 41.99\\
synthetic\_control & 91.67 & 98.33 & 98.00 & {\bf 98.67} & 98.33 & 98.00 & 90.33 & 93.00\\
uWaveGestureLibrary\_X & 63.12 & {\bf 69.96} & 67.98 & 69.71 & 68.40 & 69.40 & 63.34 & 63.18\\
uWaveGestureLibrary\_Y & 54.83 & 53.24 & 61.25 & {\bf 62.09} & 60.61 & 60.72 & 54.30 & 54.69\\
uWaveGestureLibrary\_Z & 53.74 & 60.58 & 63.34 & {\bf 64.52} & 62.53 & 63.04 & 53.38 & 53.69\\
wafer & 65.44 & 31.86 & 68.82 & 68.93 & 67.86 & {\bf 85.92} & 64.93 & 65.07\\
yoga & 49.70 & 59.97 & 57.10 & {\bf 61.70} & 54.50 & 56.23 & 46.43 & 46.43\\
\hline
\end{tabular}
\end{tiny}
\end{table}


\begin{thebibliography}{28}
\providecommand{\natexlab}[1]{#1}
\providecommand{\url}[1]{\texttt{#1}}
\expandafter\ifx\csname urlstyle\endcsname\relax
  \providecommand{\doi}[1]{doi: #1}\else
  \providecommand{\doi}{doi: \begingroup \urlstyle{rm}\Url}\fi

\bibitem[Banderier and Schwer(2005)]{banderier_2005}
Cyril Banderier and Sylviane Schwer.
\newblock Why {D}elannoy numbers?
\newblock \emph{Journal of statistical planning and inference}, 135\penalty0
  (1):\penalty0 40--54, 2005.

\bibitem[Baringhaus and Franz(2004)]{baringhaus_2004}
Ludwig Baringhaus and Carsten Franz.
\newblock On a new multivariate two-sample test.
\newblock \emph{Journal of multivariate analysis}, 88\penalty0 (1):\penalty0
  190--206, 2004.

\bibitem[Chang et~al.(2019)Chang, Huang, Sui, Fei-Fei, and Niebles]{d3tw}
Chien-Yi Chang, De-An Huang, Yanan Sui, Li~Fei-Fei, and Juan~Carlos Niebles.
\newblock D3tw: Discriminative differentiable dynamic time warping for weakly
  supervised action alignment and segmentation.
\newblock In \emph{Proc. of CVPR}, pages 3546--3555, 2019.

\bibitem[Chen et~al.(2015)Chen, Keogh, Hu, Begum, Bagnall, Mueen, and
  Batista]{UCRArchive}
Yanping Chen, Eamonn Keogh, Bing Hu, Nurjahan Begum, Anthony Bagnall, Abdullah
  Mueen, and Gustavo Batista.
\newblock The ucr time series classification archive, July 2015.
\newblock \url{www.cs.ucr.edu/~eamonn/time_series_data/}.

\bibitem[Cuturi and Blondel(2017)]{soft_dtw}
Marco Cuturi and Mathieu Blondel.
\newblock \href{https://arxiv.org/abs/1703.01541}{Soft-DTW: A differentiable
  loss function for time-series}.
\newblock In \emph{Proc. of ICML}, 2017.

\bibitem[Cuturi et~al.(2007)Cuturi, Vert, Birkenes, and Matsui]{cuturi_2007}
Marco Cuturi, Jean-Philippe Vert, Oystein Birkenes, and Tomoko Matsui.
\newblock A kernel for time series based on global alignments.
\newblock In \emph{Proc. of ICASSP}, volume~2, pages II--413. IEEE, 2007.

\bibitem[Donahue et~al.(2020)Donahue, Dieleman, Bi{\'n}kowski, Elsen, and
  Simonyan]{donahue_2020}
Jeff Donahue, Sander Dieleman, Miko{\l}aj Bi{\'n}kowski, Erich Elsen, and Karen
  Simonyan.
\newblock End-to-end adversarial text-to-speech.
\newblock \emph{arXiv preprint arXiv:2006.03575}, 2020.

\bibitem[Feydy et~al.(2019)Feydy, S{\'e}journ{\'e}, Vialard, Amari, Trouv{\'e},
  and Peyr{\'e}]{feydy_2019}
Jean Feydy, Thibault S{\'e}journ{\'e}, Fran{\c{c}}ois-Xavier Vialard, Shun-ichi
  Amari, Alain Trouv{\'e}, and Gabriel Peyr{\'e}.
\newblock Interpolating between optimal transport and mmd using sinkhorn
  divergences.
\newblock In \emph{The 22nd International Conference on Artificial Intelligence
  and Statistics}, pages 2681--2690, 2019.

\bibitem[Fr{\'e}chet(1948)]{frechet_1948}
Maurice Fr{\'e}chet.
\newblock Les {\'e}l{\'e}ments al{\'e}atoires de nature quelconque dans un
  espace distanci{\'e}.
\newblock In \emph{Annales de l'institut Henri Poincar{\'e}}, volume~10, pages
  215--310. Presses universitaires de France, 1948.

\bibitem[Genevay et~al.(2018)Genevay, Peyr{\'e}, and Cuturi]{genevay_2018}
Aude Genevay, Gabriel Peyr{\'e}, and Marco Cuturi.
\newblock Learning generative models with sinkhorn divergences.
\newblock In \emph{International Conference on Artificial Intelligence and
  Statistics}, pages 1608--1617, 2018.

\bibitem[Graves et~al.(2006)Graves, Fern{\'a}ndez, Gomez, and
  Schmidhuber]{graves_2006}
Alex Graves, Santiago Fern{\'a}ndez, Faustino Gomez, and J{\"u}rgen
  Schmidhuber.
\newblock Connectionist temporal classification: labelling unsegmented sequence
  data with recurrent neural networks.
\newblock In \emph{Proc. of ICML}, pages 369--376, 2006.

\bibitem[Hastie et~al.(2001)Hastie, Tibshirani, and Friedman]{ESLII}
Trevor Hastie, Robert Tibshirani, and Jerome Friedman.
\newblock \emph{The Elements of Statistical Learning}.
\newblock Springer New York Inc., 2001.

\bibitem[Janati et~al.(2020)Janati, Cuturi, and Gramfort]{janati_2020}
Hicham Janati, Marco Cuturi, and Alexandre Gramfort.
\newblock Spatio-temporal alignments: Optimal transport through space and time.
\newblock In \emph{Proc. of AISTATS}, pages 1695--1704. PMLR, 2020.

\bibitem[Liu and Nocedal(1989)]{lbfgs}
Dong~C Liu and Jorge Nocedal.
\newblock \href{https://doi.org/10.1007/BF01589116}{On the limited memory BFGS
  method for large scale optimization}.
\newblock \emph{Mathematical Programming}, 45\penalty0 (1):\penalty0 503--528,
  1989.

\bibitem[Luise et~al.(2018)Luise, Rudi, Pontil, and Ciliberto]{luise_2018}
Giulia Luise, Alessandro Rudi, Massimiliano Pontil, and Carlo Ciliberto.
\newblock Differential properties of sinkhorn approximation for learning with
  wasserstein distance.
\newblock In \emph{Proc. of NeurIPS}, pages 5859--5870, 2018.

\bibitem[McCallum et~al.(2012)McCallum, Bellare, and Pereira]{mccallum_2012}
Andrew McCallum, Kedar Bellare, and Fernando Pereira.
\newblock A conditional random field for discriminatively-trained finite-state
  string edit distance.
\newblock \emph{arXiv preprint arXiv:1207.1406}, 2012.

\bibitem[{Mensch} and {Blondel}(2018)]{diff_dp}
Arthur {Mensch} and Mathieu {Blondel}.
\newblock \href{https://arxiv.org/abs/1802.03676}{Differentiable dynamic
  programming for structured prediction and attention}.
\newblock In \emph{Proc. of ICML}, 2018.

\bibitem[Petitjean et~al.(2011)Petitjean, Ketterlin, and
  Gan{\c{c}}arski]{petitjean_2011}
Fran{\c{c}}ois Petitjean, Alain Ketterlin, and Pierre Gan{\c{c}}arski.
\newblock A global averaging method for dynamic time warping, with applications
  to clustering.
\newblock \emph{Pattern Recognition}, 44\penalty0 (3):\penalty0 678--693, 2011.

\bibitem[Petitjean et~al.(2014)Petitjean, Forestier, Webb, Nicholson, Chen, and
  Keogh]{petitjean_icdm}
Fran{\c{c}}ois Petitjean, Germain Forestier, Geoffrey~I Webb, Ann~E Nicholson,
  Yanping Chen, and Eamonn Keogh.
\newblock Dynamic time warping averaging of time series allows faster and more
  accurate classification.
\newblock In \emph{ICDM}, pages 470--479. IEEE, 2014.

\bibitem[Peyr{\'e} et~al.(2019)Peyr{\'e}, Cuturi, et~al.]{peyre_2019}
Gabriel Peyr{\'e}, Marco Cuturi, et~al.
\newblock Computational optimal transport: With applications to data science.
\newblock \emph{Foundations and Trends{\textregistered} in Machine Learning},
  11\penalty0 (5-6):\penalty0 355--607, 2019.

\bibitem[Ramdas et~al.(2017)Ramdas, Trillos, and Cuturi]{ramdas_2017}
Aaditya Ramdas, Nicol{\'a}s~Garc{\'\i}a Trillos, and Marco Cuturi.
\newblock On wasserstein two-sample testing and related families of
  nonparametric tests.
\newblock \emph{Entropy}, 19\penalty0 (2):\penalty0 47, 2017.

\bibitem[Saigo et~al.(2006)Saigo, Vert, and Akutsu]{saigo_2006}
Hiroto Saigo, Jean-Philippe Vert, and Tatsuya Akutsu.
\newblock Optimizing amino acid substitution matrices with a local alignment
  kernel.
\newblock \emph{BMC bioinformatics}, 7\penalty0 (1):\penalty0 246, 2006.

\bibitem[Sakoe and Chiba(1978)]{sakoe_1978}
Hiroaki Sakoe and Seibi Chiba.
\newblock Dynamic programming algorithm optimization for spoken word
  recognition.
\newblock \emph{IEEE transactions on acoustics, speech, and signal processing},
  26\penalty0 (1):\penalty0 43--49, 1978.

\bibitem[Sard{\'a}-Espinosa(2017)]{sarda_2017}
Alexis Sard{\'a}-Espinosa.
\newblock Comparing time-series clustering algorithms in r using the dtwclust
  package.
\newblock \emph{R package vignette}, 12:\penalty0 41, 2017.

\bibitem[Sulanke(2003)]{sulanke_2003}
Robert~A Sulanke.
\newblock Objects counted by the central delannoy numbers.
\newblock \emph{J. Integer Seq}, 6\penalty0 (1), 2003.

\bibitem[Sz{\'e}kely et~al.(2004)Sz{\'e}kely, Rizzo, et~al.]{szekely_2004}
G{\'a}bor~J Sz{\'e}kely, Maria~L Rizzo, et~al.
\newblock Testing for equal distributions in high dimension.
\newblock \emph{InterStat}, 5\penalty0 (16.10):\penalty0 1249--1272, 2004.

\bibitem[Tavenard et~al.(2020)Tavenard, Faouzi, Vandewiele, Divo, Androz,
  Holtz, Payne, Yurchak, Ru{\ss}wurm, Kolar, et~al.]{tavenard_2020}
Romain Tavenard, Johann Faouzi, Gilles Vandewiele, Felix Divo, Guillaume
  Androz, Chester Holtz, Marie Payne, Roman Yurchak, Marc Ru{\ss}wurm, Kushal
  Kolar, et~al.
\newblock Tslearn, a machine learning toolkit for time series data.
\newblock \emph{JMLR}, 21\penalty0 (118):\penalty0 1--6, 2020.

\bibitem[Wainwright and Jordan(2008)]{wainwright_2008}
Martin~J Wainwright and Michael~I Jordan.
\newblock
  \href{https://people.eecs.berkeley.edu/~wainwrig/Papers/WaiJor08_FTML.pdf}{Graphical
  models, exponential families, and variational inference.}
\newblock \emph{Foundations and Trends{\textregistered} in Machine Learning},
  1\penalty0 (1--2):\penalty0 1--305, 2008.

\end{thebibliography}
\end{document}